%% file: main.tex
\documentclass[10pt,twocolumn,letterpaper]{article}

 \usepackage[pagenumbers]{cvpr} 
\input{preamble}

\usepackage{indentfirst}

\usepackage[accsupp]{axessibility}

\definecolor{cvprblue}{rgb}{0.21,0.49,0.74}
\usepackage[pagebackref,breaklinks,colorlinks,citecolor=cvprblue]{hyperref}
\usepackage{amsthm}

\newtheorem{proposition}{Proposition}
\usepackage{multirow}

\newtheorem{lemma}{Lemma}
\newtheorem{definition}{Definition}

\usepackage{cuted}

\makeatletter
\@namedef{ver@everyshi.sty}{}
\makeatother
\usepackage{pgfplots}
\usepgfplotslibrary{statistics}
\usepackage{pgf,tikz}

\def\paperID{1240}
\def\confName{CVPR}
\def\confYear{2024}

\newcommand\blfootnote[1]{%
  \begingroup
  \renewcommand\thefootnote{}\footnote{#1}%
  \addtocounter{footnote}{-1}%
  \endgroup
}

\usepackage{fancyhdr}
\usepackage{setspace}

\title{Finsler-Laplace-Beltrami Operators with Application to Shape Analysis}

\author{Simon Weber\textsuperscript{1,2*} \\ 
{\tt\small sim.weber@tum.de} \hspace{-1em}
\and Thomas Dag\`es\textsuperscript{3*} \\ 
{\tt\small thomas.dages@cs.technion.ac.il} \hspace{-1em}
\and Maolin Gao\textsuperscript{1,2} \\ 
{\tt\small maolin.gao@tum.de} \hspace{-1em}
\and Daniel Cremers\textsuperscript{1,2} \\ 
{\tt\small cremers@tum.de}
}

\begin{document}

\maketitle
\thispagestyle{fancy}

\begin{strip}%
\centering
\includegraphics[width=1.9\columnwidth]{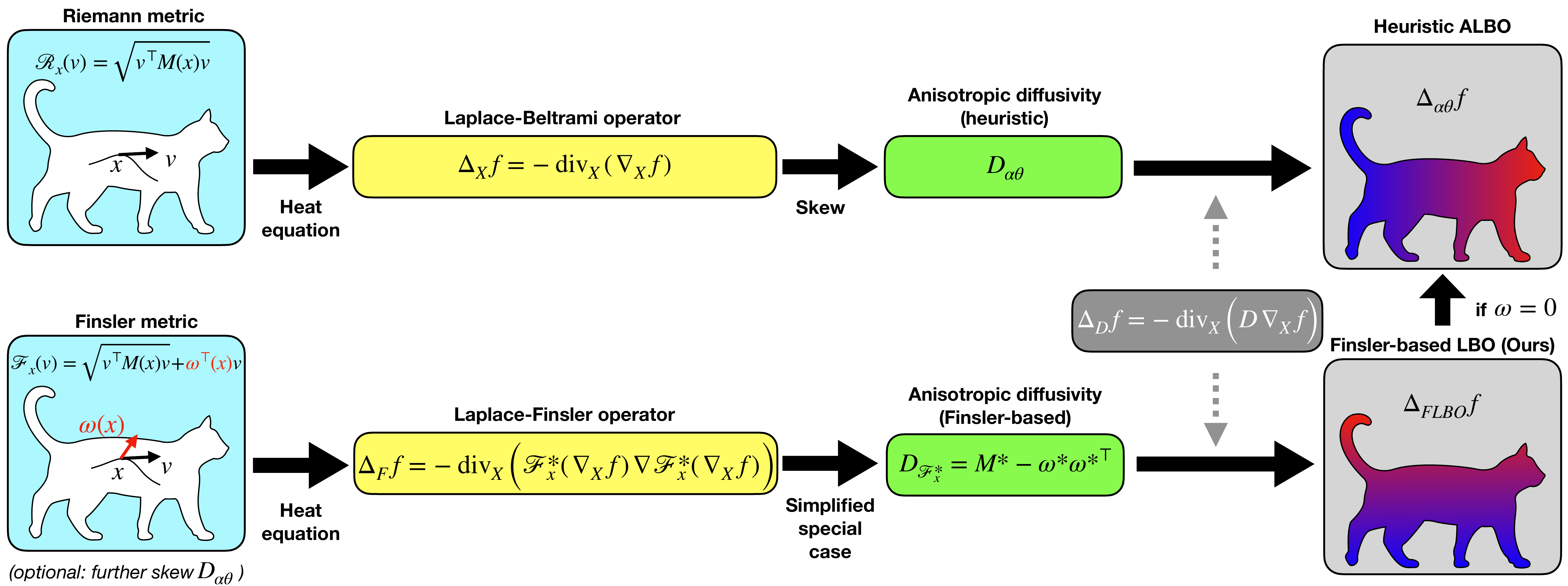}
   \captionof{figure}{Simplified overview of the derivation of the
   Finsler-Laplace-Beltrami operator (FLBO), built from a Finsler metric. It generalizes the traditional heuristic anisotropic Laplace-Beltrami operators (ALBO) \cite{boscaini2016,li2020} (top).
   The FLBO can directly replace the ALBOs in surface processing tasks like shape correspondence, notably by constructing shape-dependent anisotropic convolutions in their spectral domain. 
   }
   \label{fig: overview theory teaser}
\end{strip}

\input{sec_finsler/0_abstract}

\input{sec_finsler/1_intro}

\input{sec_finsler/2_related_works}

\input{sec_finsler/3_background}

\input{sec_finsler/5_finsler_heat_kernel}

\input{sec_finsler/6_experiments}

\input{sec_finsler/7_conclusion}

{
    \small
    \bibliographystyle{ieeenat_fullname}
    \bibliography{main}
}

\input{sec_finsler/9_appendix_for_supp_mat}

\end{document}

%% file: preamble.tex
\usepackage[dvipsnames]{xcolor}

\DeclareMathOperator{\diverg}{div}

\usepackage{cases}

%% file: sec_finsler/0_abstract.tex
\begin{abstract}
\blfootnote{* equal contribution.} 
\footnotetext[1]{Technical University of Munich}
\footnotetext[2]{Munich Center for Machine Learning}
\footnotetext[3]{Technion - Israel Institute of Technology}
The Laplace-Beltrami operator (LBO) emerges from studying manifolds equipped with a Riemannian metric. It is often called the {\em swiss army knife of geometry processing} as it allows to capture intrinsic shape information and gives rise to heat diffusion, geodesic distances, and a multitude of shape descriptors.  It also plays a central role in geometric deep learning.
In this work, we explore Finsler manifolds as a generalization of Riemannian manifolds.
We revisit the Finsler heat equation and derive a Finsler heat kernel and a Finsler-Laplace-Beltrami Operator (FLBO): a novel theoretically justified anisotropic Laplace-Beltrami operator (ALBO). In experimental evaluations we demonstrate that the proposed FLBO is a valuable alternative to the traditional Riemannian-based LBO and ALBOs for spatial filtering and shape correspondence estimation. We hope that the proposed Finsler heat kernel and the FLBO will inspire further exploration of Finsler geometry in the computer vision community.

\end{abstract}

%% file: sec_finsler/1_intro.tex
\section{Introduction}
\label{sec:intro}

\subsection{Finsler versus Riemann}

In his PhD thesis of 1918, Paul Finsler introduced the notion of a differentiable manifold equipped with a Minkowski norm on each tangent space \cite{finsler1918ueber}. It constitutes a generalization of Riemannian manifolds \cite{riemann1854uber} that Elie Cartan later referred to as {\em Finsler manifolds} \cite{cartan1933espaces}.  While Riemannian manifolds are omni-present in computer vision and machine learning, Finsler manifolds have remained largely unexplored.

Essentially, a Finsler metric is the same as a Riemannian one without the quadratic assumption. The consequence is that the metric not only depends on the location, but also on the tangential direction of motion. This implies that the length of a curve between two points $A$ and $B$ may change whether we are traversing it from $A$ to $B$ or the other way round, thus geodesic curves from $A$ to $B$ may differ from those from $B$ to $A$. \begin{figure}[ht]
  \centering
   \includegraphics[width=0.35\columnwidth]{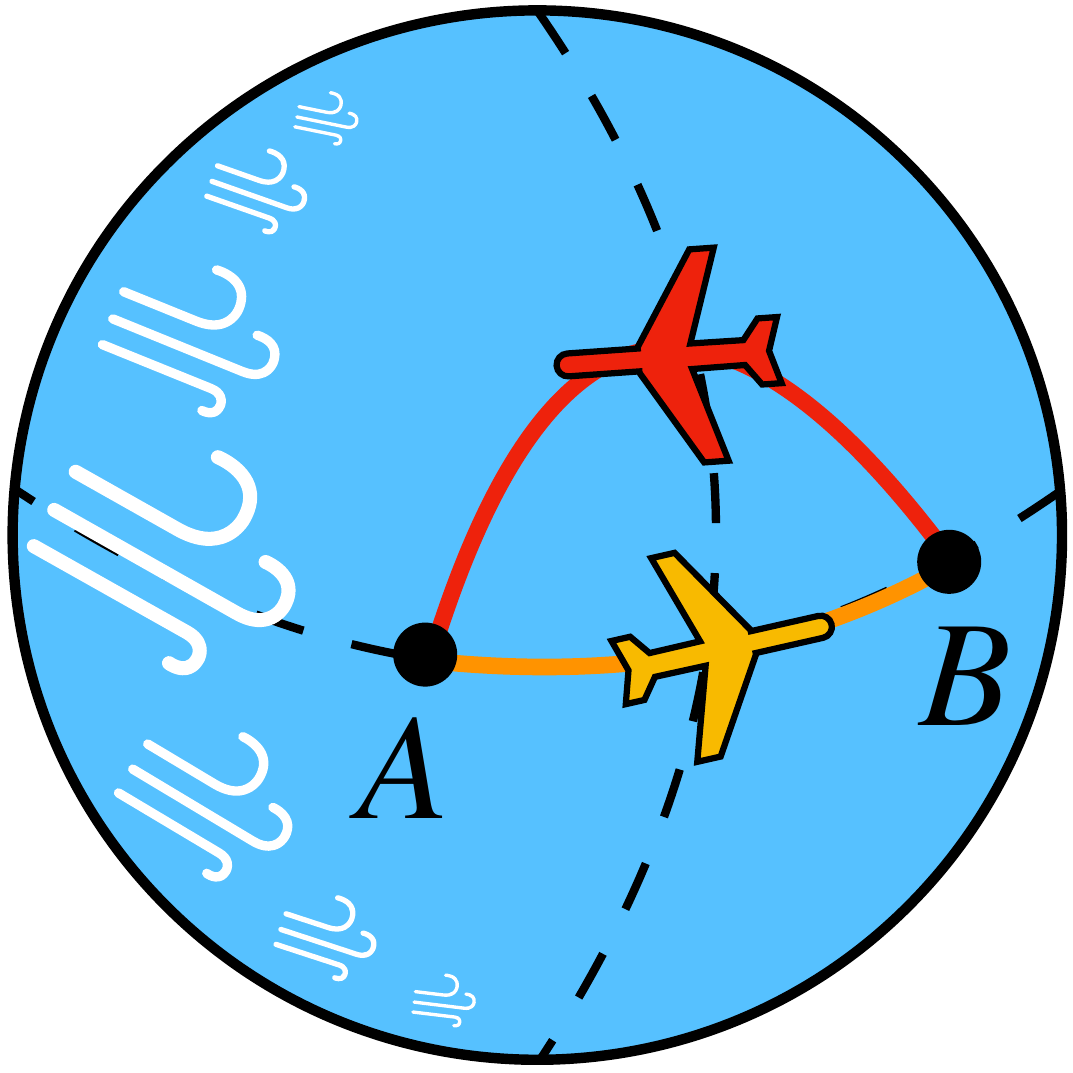}
   \caption{Due to non-uniform wind currents, the geodesic curve from $A$ to $B$ differs from the one from $B$ to $A$. While not possible in Riemannian geometry, such asymmetries are characteristic of Finsler manifolds.}
   \label{fig: finsler planes}
\end{figure}
This means that on a Finsler manifold, there is a built-in asymmetric anisotropy, which contrasts with the symmetric Riemannian case. 
In the Riemannian world, moving from an isotropic to a non isotropic metric allows more flexibility to capture interesting features on the manifolds. For instance, on a human hand shape, we may want a point on a finger to be close in some sense to all the points along the ring of the finger more than those along the length of the finger. Shifting from a Riemannian world to a Finsler world allows to increase the flexibility of anisotropy by also allowing non-symmetric distances. For many physical systems the Finsler metric is actually the natural choice of distance. For instance, a boat on a river or a plane in the wind might find that a path along the maximum current provides the shortest travel time whereas the return voyage might require moving away from the maximum current --  see \cref{fig: finsler planes}.

\subsection{Contribution}

The aim of this paper is to explore Finsler manifolds, study the Finsler heat diffusion, introduce a Finsler-Laplace-Beltrami Operator (FLBO), and demonstrate its applicability to shape analysis and geometric deep learning via the shape matching problem. In particular:

\begin{itemize}
    \setlength\itemsep{0.7em}
    \item We present a theoretical exploration of Finsler manifolds in a concise and self-contained form that is accessible to the computer vision and machine learning community.
    \item We revisit the heat diffusion on Finsler manifolds and derive a tractable equation.
    \item We generalize the traditional (anisotropic) Laplace-Beltrami operator
    to a Finsler-Laplace-Beltrami operator (FLBO).
    \item By leveraging the relationship between the diffusivity coefficient and the Riemannian metric, we propose an easy-to-implement discretization of the FLBO. 
    \item We perform extensive evaluation of the proposed approach for shape correspondence based on the FLBO. We highlight the benefits of this new operator in terms of accuracy on both full and partial shape datasets.
    \item We release our code as open source to facilitate further research: \url{https://github.com/tum-vision/flbo}.
\end{itemize}

%% file: sec_finsler/2_related_works.tex
\section{Related work}
\label{sec:related_works}
The research domains most relevant to our work are  geometric deep learning for shape correspondence and the use of the Finsler metric in computer vision.

\subsection*{Geometric deep learning for shape correspondence}

Finding correspondences between shapes is a fundamental task in computer vision and a key component of many geometry processing applications, such as object recognition, shape reconstruction, texture transfer, or shape morphing. Initially solved with handcrafted approaches \cite{vestner2017efficient,chen2015robust,ren2018continuous},
the shape correspondence problem has recently benefited from the success of neural networks
\cite{lecun1998lenet,krizhevsky2012alexnet,simonyan2014vgg,szegedy2016inception,he2016resnet,goodfellow2016deep} to learn features achieving tremendous results. 

Initially seen in the Euclidean space, shapes are nowadays mostly considered as embedded manifolds. Intrinsic information, such as geodesic distances, are products of the heat diffusion paradigm, which has mostly been studied on manifolds equipped with a Riemannian metric. For such manifolds, heat diffusion is governed by the Poisson equation, and thus naturally the Laplace-Beltrami operator (LBO) plays a central role \cite{bronstein2008numerical}. 
The LBO is traditionally isotropic on the manifold, but recent approaches have heuristically injected anisotropy into it, producing anisotropic LBOs (ALBO) \cite{andreux2014}. Coupled with spatial filtering \cite{li2020} ALBOs have shown impressive accuracy in shape correspondence. 

We here focus on methods centred around anisotropic ALBOs. Nevertheless, many other types of approaches exist for full and partial shape correspondence. In particular, we point out the vast literature based on the functional map framework \cite{ovsjanikov2012functional,ovsjanikov2016computing,litany2017dfm,halimi2019unsupervised,bracha2020shape,donati2020deepGeoMaps,sharma2020weakly,attaiki2021dpfm,cao2023unsupervised,bracha2023partial}.

In \cite{litman2013learning}, the well-known heat kernel signature and wave kernel signature are generalized by leveraging the spectral properties of deformable shapes and learning task-specific spectral filters. 
This work is extended in \cite{boscaini2016cremers} to direction-sensitive feature descriptors by using the eigendecomposition of anisotropic Laplace-Beltrami Operators (ALBO) \cite{andreux2014}. Complementary to such spectral approaches, a line of works builds on spatial filtering methods.
The idea is to construct patch operators directly on manifolds or graphs in the spatial domain. 
In \cite{Mitchel_2021_ICCV}, a convolution operator acting on vector fields without constraints on the filters themselves is defined. In \cite{boscaini2016}, a local intrinsic representation of a function defined on a manifold is extracted with anisotropic heat kernels as spatial weighting functions. 
This idea is extended in \cite{li2020} by introducing an anisotropic convolution operator in the spectral domain. 
In this work, the eigendecomposition of multiple ALBOs proposed in \cite{andreux2014}
is filtered via Chebychev polynomials with learnable coefficients. This method achieves impressive results for shape correspondence in terms of accuracy.

\subsection*{Finsler computer vision}

Computer vision on Finsler manifolds, a natural generalization of Riemannian manifolds, is a largely uncharted research area. 
In \cite{ratliff2021}, robotics applications are pushed beyond Riemannian geometry and the Euler-Lagrange equation is reviewed from a Finsler point of view.
A few works in image segmentation use a Finsler metric, either to add the orientation of an image as an extra spatial dimension \cite{chen2016minpath}, or to reformulate geodesic evolutions for region-based active contours \cite{chen2017elastica,chen2016geoevo}. The heat equation on Finsler manifolds has been defined in mathematics \cite{ohta2009}. In \cite{yang2019}, a geodesic distance based on Finsler metrics is applied to contour detection. Nevertheless the chosen heat equation is theoretically unmotivated and is solely addressed from a heuristic point of view.

%% file: sec_finsler/3_background.tex
\section{Background}
\label{sec:background}

We traditionally model a 3D shape as a two-dimensional compact Riemannian manifold $(X, M)$, with $X$ a smooth manifold equipped with an inner product $\langle u,v\rangle_M = u^\top M(x) v$ for vectors $u,v$ on the tangent space $T_x X$ at each point $x\in X$.
We alternatively model the shape as a Finsler manifold  $(X,\mathcal{F}_x)$ equipped with a Randers metric $\mathcal{F}_{x}: T_{x}X \rightarrow \mathbb{R}$ at point $x$. This manifold can be viewed as a generalization of the Riemannian model, as the metric does not necessarily induce an inner product. We point out that we recover a Riemannian manifold if $\mathcal{F}_{x}$ is symmetric.
In this section we briefly introduce the Riemannian heat diffusion, the Randers metric and the Finsler heat diffusion.

\subsection{Riemannian heat diffusion}
\paragraph{Laplace-Beltrami operator.}
The Laplace-Beltrami Operator (LBO) on the Riemannian shape $X$ is defined as:
\begin{equation}
\Delta_{X}f(x) = -\diverg_{X}(\nabla_{X}f(x)),
\end{equation}
with $\nabla_{X}f$ and $\diverg_{X}f$ being the intrinsic gradient and divergence of $f(x) \in L^{2}(X)$. The LBO is a positive semi-definite operator with a real eigen-decomposition $(\lambda_{k},\phi_{k})$:
\begin{equation}
\Delta_{X}\phi_{k}(x) = \lambda_{k}\phi_{k}(x),
\end{equation}
with countable eigenvalues $0 = \lambda_{0} \leq \lambda_{1} \leq \ldots$ growing to infinity. The corresponding eigenfunctions $\phi_{0},\phi_{1}\ldots$ are orthonormal for the standard inner product on $L^2(X)$:
\begin{equation}
\langle \phi_{i}, \phi_{j} \rangle = \delta_{ij},
\end{equation}
and form an orthonormal basis for $L^{2}(X)$. It follows that any function $f \in L^{2}(x)$ can be expressed as:
\begin{equation}
f(x) = \sum_{k \geq 0}\langle \phi_k, f \rangle \phi_{k}(x) = \sum_{k \geq 0} \hat{f}(\lambda_{k}) \phi_{k}(x),
\end{equation}
and $\hat{f}(\lambda_{k}) = \langle \phi_{k}, f \rangle
$ are the coefficients of the manifold Fourier transform of $f$.
The Convolution Theorem on manifolds is given by:
\begin{equation}\label{eq:convolution_iso}
(f \ast g) (x) = \sum_{k \geq 0} \hat{f}(\lambda_{k}) \hat{g}(\lambda_{k}) \phi_{k}(x) ,
\end{equation}
where $f*g$ is the convolution of signal $f$ with filter kernel $g$ and $\hat{f}(\lambda_{k}) \hat{g}(\lambda_{k})$ is the spectral filtering of $f$ by $g$.

\paragraph{Heat diffusion.}
The LBO governs the diffusion equation:
\begin{equation}\label{eq:heat}
\frac{\partial f(x,t)}{\partial t} = - \Delta_{X}f(x,t) , 
\end{equation}
where $f(x,t)$ is the temperature at point $x$ at time $t$. Given some initial heat distribution $f_{0}(x) = f(x,0)$ the solution of \cref{eq:heat} at time $t$ is given by the convolution:
\begin{equation}
f(x,t) = \int_{X}f_{0}(\xi) h_{t}(x,\xi) d\xi ,
\end{equation}
with the heat kernel $h_{t}(x,\xi)$. In the spectral domain it is expressed as:
\begin{equation}
h_{t}(x,\xi) = \sum_{k \geq 0} e^{-t\lambda_{k}}\phi_{k}(x)\phi_{k}(\xi) .
\end{equation}
\paragraph{Anistropic LBO.}
Extending the traditional LBO presented above, an anistropic LBO (ALBO) is created by changing the diffusion speed along chosen directions on the surface, typically the principal directions of curvature \cite{andreux2014}: 
\begin{equation}
\Delta_\alpha f(x) = -\diverg_{X}(D_\alpha(x) \nabla_{X}f(x)) ,
\end{equation}
where $D_\alpha(x)$ is a thermal conductivity tensor acting on the intrinsic gradient direction in the tangent plane and $\alpha$ is a scalar expressing the level of anisotropy. The ALBO can be extended to multiple anisotropic angles $\theta$ \cite{boscaini2016}:
\begin{equation}
\Delta_{\alpha \theta}f(x) = - \diverg_{X}(D_{\alpha \theta}(x) \nabla_{X}f(x)) ,
\end{equation}
by defining the thermal conductivity tensor:
\begin{equation}
D_{\alpha \theta}(x) = R_{\theta}D_{\alpha}(x)R_{\theta}^{\top} ,
\end{equation}
where $R_{\theta}$ is a rotation in the tangent plane (around surface normals) with angle $\theta$.

\paragraph{Anistropic convolution operator}
Similarly, \cref{eq:convolution_iso} is extended by introducing the anisotropic convolution operator \cite{li2020}:
\begin{equation}
(f \ast g)_{\alpha \theta}(x) = \int_{y\in X}f(y) g_{\alpha \theta, x}(y)dy ,
\end{equation}
where $g_{\alpha \theta, x}(y)$ is an oriented kernel function centered at point $x$ from the spectral domain

\begin{equation}
    \label{eq: anisotropic conv kernel spectral alpha theta}
    g_{\alpha \theta, x}(y) = \sum_{k \geq 0}\hat{g}(\lambda_{\alpha \theta,k}) \phi_{\alpha \theta, k}(x) \phi_{\alpha \theta, k}(y) .
\end{equation}
A straightforward derivation shows that:
\begin{equation}
    \label{eq: anisotropic conv alpha theta}
    (f \ast g)_{\alpha \theta}(x) = \sum_{k \geq 0} \hat{g}(\lambda_{\alpha \theta,k}) \hat{f}(\lambda_{\alpha \theta, k}) \phi_{\alpha \theta,k}(x) ,
\end{equation}
where $\hat{f}(\lambda_{\alpha \theta,k}) = \langle \phi_{\alpha \theta,k} , f\rangle$ are the coefficients of the anisotropic Fourier transform of $f$. To extract features along all directions, the anisotropic convolution operators are summed and then:
\begin{equation}
    \label{eq: anisotropic conv alpha}
    (f \ast g)_\alpha(x) = \int_{0}^{2 \pi} (f \ast g)_{\alpha \theta}(x) d\theta .
\end{equation}
The design of $\hat{g}(\lambda)$ is crucial when defining the anisotropic convolution operators. Following \cite{defferrard2016}, Chebychev polynomials are adopted to define polynomial filters $\hat{g}(\lambda)$ in \cite{li2020}.

\subsection{Randers metric}

Riemannian manifolds are equipped with a quadratic metric $\mathcal{R}_x(v) = \lVert v\rVert_{M(x)}$ that naturally induces a positive-definite inner product on the tangent space $T_xM$ at each point $x$. Finsler geometry is often said to be the same as Riemannian geometry without the quadratic constraint. A Finsler metric thus no longer induces an inner-product and is not required to be symmetric. Formally, a Finsler metric is given by the following properties.

\begin{definition}[Finsler metric]
    A Finsler metric $\mathcal{F}_x:T_xX\to \mathbb{R}_+$ at point $x$ is a smooth metric satisfying the triangle inequality, positive homogeneity, and positive definiteness axioms: 
    \begin{numcases}{}
        \mathcal{F}_x(v + v') \le \mathcal{F}_x(v) + \mathcal{F}_x(v'),\;& \hspace{-1em}$\forall v,v'\in T_xX$\\
        \mathcal{F}_x(\gamma v) = \gamma \mathcal{F}_x(v), \;& \hspace{-1em}$\forall (\gamma, v)\!\in\! \mathbb{R}_+\!\times\! T_xX$ \\
        \mathcal{F}_x(v) = 0 \iff v=0.
    \end{numcases}
\end{definition}

We consider Randers metrics, one of the simplest type of Finsler metrics generalizing the Riemannian one.

\begin{definition}[Randers metric] 
    Let $M : X \to S_{d}^{++}$ be a tensor field  of symmetric positive definite matrices and a vector field $\omega : X \to \mathbb{R}^{d}$ such that $\lVert \omega(x) \rVert _{M(x)^{-1}} < 1$ for all $x\in X$. The Randers metric associated to $(M, \omega)$ is given by
    \begin{equation}
        \label{eq: randers}
        \mathcal{F}_{x}(v) = \lVert v \rVert_{M(x)} + \langle \omega(x), v \rangle, \quad\forall x\in X, v \in T_{x}X.
    \end{equation}
\end{definition}

The Randers metric $\mathcal{F}$ is the sum of the Riemannian metric associated to $M$ and an asymmetric linear part governed by $\omega$. Notably when $\omega \equiv 0$, $\mathcal{F}$ is a traditional Riemannian metric. 
The condition $\lVert \omega(x) \rVert _{M(x)^{-1}} < 1$ ensures the positivity of the metric (see supplementary material).

From any metric we can construct a new metric associated to it, its dual, which plays a central role in differential geometry.

\begin{definition}[Dual Finsler metric]
    \label{def: dual finsler}
    The dual metric of a Finsler metric $\mathcal{F}_x$ is the metric $\mathcal{F}_x^*:T_xX\to\mathbb{R}_+$ defined by 
    \begin{equation}
        \mathcal{F}_x^{*}(v) = \max\{\langle v, b \rangle ; b \in \mathbb{R}^{d}, \mathcal{F}_x(b) \leq 1\}.
    \end{equation}
\end{definition}

When considering only Randers metrics, a subtype of Finsler metrics, the dual metric becomes explicit.

\begin{lemma}[Dual Randers metric]
Let $\mathcal{F}_x$ be a Randers metric given by $(M,\omega)$. The dual metric of $\mathcal{F}_x$ is also a Randers metric $\mathcal{F}_x^{*}$ associated to $(M^*, \omega^*)$ that satisfies
\begin{equation}
\label{eq: definition matrix M* w*}
\begin{pmatrix} \alpha M^{*} & \omega^{*} \\
{\omega^*}^{\top} & \frac{1}{\alpha} \\
\end{pmatrix} = \begin{pmatrix}
M & \omega \\
\omega ^{\top} & 1 \\
\end{pmatrix}^{-1},
\end{equation} 
where $\alpha = 1 - \langle \omega, M^{-1} \omega \rangle > 0$.
\end{lemma}

A proof is provided in the supplementary material.

From \cref{eq: definition matrix M* w*} we can derive: 
\begin{equation}
\label{eq: M* formula}
M^{*} = \frac{1}{\alpha^2}\Big(\alpha M^{-1} + (M^{-1}\omega)(M^{-1}\omega)^{\top}\Big),
\end{equation}
and 
\begin{equation}
\label{eq: omega* formula}
\omega^{*} = -\frac{1}{\alpha}M^{-1}\omega,.
\end{equation}

It is simple to show that $\lVert \omega^*\rVert_{{M^*(x)}^{-1}} <1$ (see supplementary material). The dual Randers metric $\mathcal{F}_x^*$ is the Randers metric on the manifold $X$ associated to the new Riemannian and linear components $(M^*,\omega^*)$:
\begin{equation}
    \label{eq: F* randers formula}
    \mathcal{F}_x^*(v) = \lVert v \rVert_{M^*(x)} + \langle w^*(x), v\rangle.
\end{equation}

\subsection{Finsler heat equation}

The classical heat equation in the Riemannian case is the gradient flow of the Dirichlet energy, and by analogy we can derive a Finsler heat equation as the gradient descent of the dual energy $\tfrac{1}{2}\mathcal{F}_x^*(v)^2$, in the following way \cite{bonnans2022linear,ohta2009}.

\begin{definition}[Finsler heat equation]
Given a manifold $X$ equipped with a Finsler metric $\mathcal{F}_x$, the heat equation on this Finsler manifold for smooth vector fields 
$u(x,t)$ with $(x,t)\in X\times \mathbb{R}_+$
is given by the gradient descent of the dual energy
\begin{equation}
\label{eq: finsler heat equation}
\partial_{t}u = \diverg(\mathcal{F}_x^*(\nabla u) \nabla \mathcal{F}_x^*(\nabla u)) .
\end{equation}

\end{definition}

%% file: sec_finsler/5_finsler_heat_kernel.tex
\section{Finsler heat kernel}
\label{sec:finsler_heat_kernel}

We now analyse the Finsler heat equation to derive a novel Laplacian operator for shape matching. Our plan is the following:
\begin{itemize}
    \setlength\itemsep{0.7em}

    \item Analyse the heat equation on a Finsler manifold equipped with a Randers metric,

    \item From a specific case where the solution is explicit, we exhibit a simplified equation behaving like a regular heat equation with an external heat source,

    \item The solution to this new problem is explicit and governed by the homogeneous term, which is an anisotropic heat diffusion 
    with diffusivity involving
    the Randers metric,

    \item The homogeneous solution is given by a convolution with an anisotropic kernel 
    depending on the Randers metric,

    \item This heat kernel allows the construction of a new Finsler-Laplace-Beltrami operator (FLBO) that benefits from all the properties of ALBOs.
\end{itemize}

\subsection{Finsler heat diffusion}

The heat equation is well-known to be related to distances and geodesics in short time, and thus it is often analyzed when $t\to 0$. We look at the following specific case in short time. To provide a tractable solution, we also here assume that we slightly depart from Riemannian geometry by taking small $\lVert\omega\rVert$.

\begin{proposition}[Simplified Randers heat equation - distance trick]
    \label{prop: heat eqn simplified randers distance trick}
    If the initial solution $u_0 = u(x,0)$ of the heat equation 
    satisfies $\mathcal{F}_x^*(\nabla u_0) = 1$,
    then in short time the Finsler heat equation on a Randers metric associated to $(M,\omega)$ with small $\lVert\omega\rVert$ simplifies to
    \begin{equation}\label{eq: finsler_heat_equation}
        \partial_{t}u = \diverg\left((M^{*}-\omega^{*}{\omega^{*}}^{\top})\nabla u\right) + \diverg(\omega^{*}) \, .
    \end{equation}
\end{proposition}

\begin{proof}
    We provide the proof in the supplementary material. 
\end{proof}

The initial condition $\mathcal{F}_x^*(\nabla u_0) = 1$ holds when $-u_0$ is a distance function, as the Finsler Eikonal equation is given by $\mathcal{F}_x^*(-\nabla u) = 1$ \cite{mirebeau2014}. 
In practice, the solution to \cref{eq: finsler_heat_equation} is not our primary concern.
We only use this particular initialisation to exhibit a tractable solution that allows us to define a family of Finsler-based LBOs. 
We name the coefficient $D_{\mathcal{F}_x^*} = M^{*} - \omega^{*}{\omega^{*}}^{\top}$ as the Finlser diffusivity.
We can interpret \cref{eq: finsler_heat_equation} as a homogeneous heat equation:
\begin{equation}\label{eq: homogeneous finsler}
\partial_{t}u(x,t) = \diverg(D_{\mathcal{F}_x^*}\nabla u(x,t)),
\end{equation}
to which we add an external heat source $\diverg(\omega^{*}(x))$. 
Given some initial heat condition $u_{0}(x) = u(x,0)$, the solution of \cref{eq: finsler_heat_equation} is the sum of the solution of \cref{eq: homogeneous finsler},
\begin{equation}
u(x,t) = \int_{X} u_{0}(\xi)h_{t}(x,\xi)d\xi,
\end{equation}
where $h_t$ is the
heat kernel for \cref{eq: homogeneous finsler}, 
and of any particular solution, such as
\begin{equation}
u^{d}(x,t) = \int_{X} \diverg(\omega^{*}(\xi)) \hat{h}_{t}(x,\xi) d\xi \,,
\end{equation}
where $\hat{h}$ is the average over time of the heat kernel $h_{t}$:
\begin{equation}
\hat{h}_{t}(x,\xi) = \int_{s=0}^{t} h_{(t-s)}(x,\xi) ds \, .
\end{equation}

\subsection{Finsler-Laplace-Beltrami operator}

In Riemannian geometry, the heat kernel is by definition given by the convolution kernel of the solution of the homogeneous equation. An external source introduces a convolution with the time average of the heat kernel. By analogy, 
from the homogeneous \cref{eq: homogeneous finsler} we define the Finsler-Laplace-Beltrami operator (FLBO) as:
\begin{equation}
\Delta_{\mathit{FLBO}}u(x) = -\diverg_X(D_{\mathcal{F}_x^*} \nabla_X u(x)) \, .
\end{equation}

The FLBO is a special case of an ALBO on a Riemannian manifold, that was derived after analyzing the heat equation when the manifold is equipped with a Finsler metric instead. Unlike previous ALBOs, it is theoretically justified rather than empirically designed based solely on heuristics. Since the FLBO in an ALBO, it inherits all of its properties, such as the spectral decomposition of a compact positive semi-definite operator
\begin{equation}
    \Delta_{\mathit{FLBO}}\phi_k^{\mathit{FLBO}}(x) = \lambda_k^{\mathit{FLBO}}\phi_k^{\mathit{FLBO}}(x).
\end{equation}
The FLBO also differs from the Laplace-Finsler operator (LF) deriving from \cref{eq: finsler heat equation}:
\begin{equation}
    \Delta_F u(x) = -\diverg(\mathcal{F}_x^*(\nabla u)\nabla \mathcal{F}_x^*(\nabla u)),
\end{equation}
In particular, the LF operator acts on a Finsler manifold and not a Riemannian one like the FLBO. Thus, it cannot be used as a direct replacement of LBOs in most applications.

In the previous section we have shown that the solution of the simplified Randers heat diffusion equation exhibits the Finsler heat kernel $h_{t}(x,\xi)$ associated with the homogeneous Finsler heat diffusion equation. 
\begin{proposition}\label{prop_finsler_heat_kernel} 
In the spectral domain, the Finsler heat kernel can be expressed as
\begin{equation}
h_{t}^{FLBO}(x,\xi) = \sum_{k \geq 0} e^{-t \lambda_{k}^{FLBO}} \phi_{k}^{FLBO}(x) \phi_{k}^{FLBO}(\xi).
\end{equation}
\end{proposition}

We have unified the Finsler heat diffusion equation and the anisotropic Laplace-Beltrami operator formulation. The link comes from the Finsler heat kernel that was exhibited in the explicit solution of the simplified equation \cref{eq: finsler_heat_equation} in a special case. We point out that the asymmetric component $\omega$ of the Randers metric influences the diffusivity coefficient and then the eigendecomposition of the Finsler heat kernel.

%% file: sec_finsler/6_experiments.tex
\section{Experiments}
\label{sec:experiments}

\subsection{Discretization}
We present a possible discretization of our FLBO given a manifold discretized by sampling $n$ vertices.
As the FLBO is a positive semi-definite operator, and by analogy with the LBO it can be approximated by an $n \times n$ sparse matrix $L_{\mathit{FLBO}} = - S_{\mathit{FLBO}}^{-1} W_{\mathit{FLBO}}$ with the mass matrix $S_{\mathit{FLBO}}$ and the stiffness matrix $W_{\mathit{FLBO}}$.
Our FLBO is not restricted to a particular discrete surface representation. Following \cite{boscaini2016,li2020}, we focus on triangular meshes to discretize the Finsler metrics, but representations for other discretizations can be easily deduced.
As in \cite{boscaini2016,li2020}, we introduce various angles $\theta$, each giving a Finsler metric $(M_\theta, \omega_\theta)$, to produce convolution operators from FLBOs sensitive to $\theta$, and then sum them up to get a final Finsler-based convolution operator incorporating features from all directions. We use the same notations as in \cite{boscaini2016,li2020}.

Our manifold is discretized as a triangular mesh with vertices $V$, edges $E$, and faces $F$. For each triangle $ijk\in F$, we define $U_{ijk} = (\hat{u}_{M}, \hat{u}_{m}, \hat{n})$ as an orthonormal reference frame associated to the face unit normal $\hat{n}$ and the directions of principal curvature $\hat{u}_{M}, \hat{u}_{m} \in \mathbb{R}^{3}$. 
We denote $\hat{e}_{ij}\in \mathbb{R}^3$ the unit vector along the embedding of the edge $(i,j)\in E$ pointing from $i$ to $j$, and $\alpha_{ij}$ and $\beta_{ij}$ the angles at $k$ and $h$ of the adjacent triangles $ijk\in F$ and $ijh\in F$ respectively. We write $R_\theta\in \mathbb{R}^3$ the rotation matrix of angle $\theta$ around the third basis vector (to be identified as $\hat{n}$).
The notations match those of \cref{fig: notations boscaini}.

\begin{figure}[ht]
  \centering
   \includegraphics[width=0.65\columnwidth]{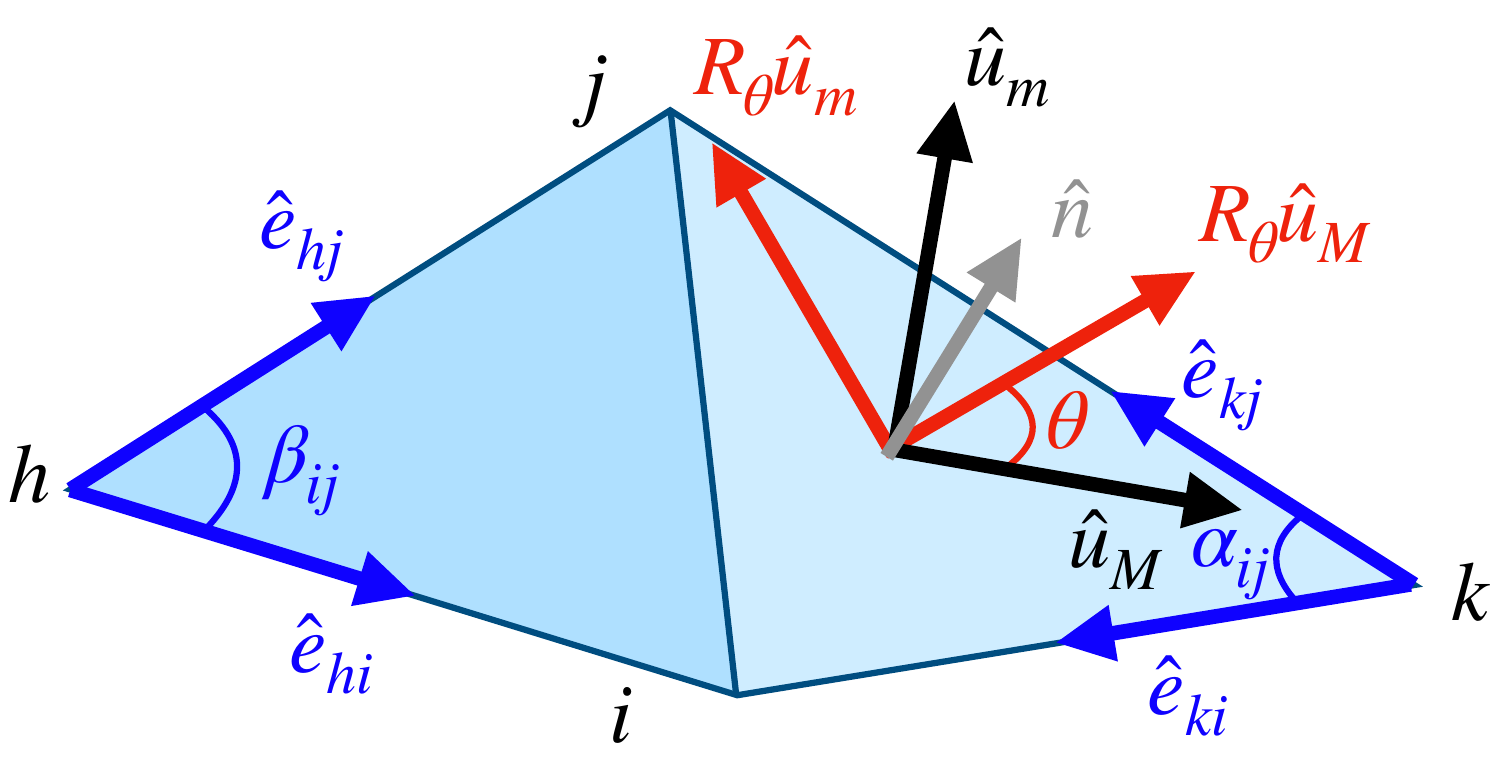}
   \caption{Notations for unit vectors and angles on the triangular mesh. The figure mimics the illustration in \cite{boscaini2016}.}
   \label{fig: notations boscaini}
\end{figure}

In all our experiments, the mass matrix $S_{\mathit{FLBO}}$ is computed as a diagonal matrix with diagonal element $i$ equal to one third of the areas of faces having $i$ as a vertex \cite{desbrun1999implicit}.

Our chosen Riemannian metric is also anisotropic using ideas from the ALBO community. When $\omega \equiv 0$ we obtain the traditional ALBO \cite{li2020}.
The diffusivity $D_\alpha(x)$ of ALBOs is usually created by introducing a scaling factor $\alpha>0$ into one of its eigenvalues by various heuristic formulae to break the isotropy along eigendirections \cite{boscaini2016,li2020}. We firstly build on the strategy of \cite{li2020} to create the uniform anisotropic Riemannian diffusivity
\begin{equation}
    D_\alpha = 
    \begin{pmatrix} 
        \frac{1}{1+\alpha} & \\
         & 1
    \end{pmatrix},
\end{equation}
with $\alpha = 10$ in our experiments as in \cite{li2020}.
For each angle $\theta$, we rotate this diffusivity to create the oriented diffusivity matrix $H_{\alpha\theta}$, also called shear matrix,
\begin{equation}
    \label{eq: H = R U D U^T R^T}
    H_{\alpha\theta} = R_{\theta} U_{ijk} D_\alpha U_{ijk}^{\top} R_{\theta}^{\top}.
\end{equation}
We can now define our anisotropic Riemannian metric at angle $\theta$. As we want our construction to generalise the Riemannian case $\omega\equiv0$, and since  $D_{\mathcal{F}_x^*}$ equals $M^{-1}(x)$ when $\omega\equiv0$, we choose
\begin{equation}
    \label{eq: M = H^-1}
    M_{\alpha\theta} = H_{\alpha\theta}^{-1}.
\end{equation}
Note that our scheme recovers the isotropic Riemannian metric  when $\alpha = 0$ as the shear matrix is then the identity. For each angle $\theta$, we introduce the Randers anisotropic component rotated by $\theta$ from the direction of maximal curvature
\begin{equation}
    \label{eq: omega = tau R u}
    \omega_{\theta} = \tau R_{\theta} \hat{u}_{M},
\end{equation}
where $\tau$ is a hyperparameter to ensure $\lVert \omega\rVert_{M_{\alpha\theta}^{-1}}<1$. We have thus constructed the Randers metric $(M_{\alpha\theta},\omega_\theta)$, from which we can easily compute its dual $(M_{\alpha\theta}^*,\omega_\theta^*)$ and its Finsler diffusivity $D_{\mathcal{F}_x^*}^{\alpha\theta}$. We then use the generalisation of the classical cotangent weight scheme for anisotropic diffusivities \cite{boscaini2016} to compute the weights of $W_{\mathit{FLBO}}^{\alpha\theta}$:
\begin{equation}
    \label{eq: weights w flbo theta}
    w_{ij} = 
    \begin{cases}
        \frac{1}{2} \left(\frac{ \langle \hat{e}_{kj}, \hat{e}_{ki} \rangle_{D_{\mathcal{F}_x^*}^{\alpha\theta}}}{\sin \alpha_{ij}} + \frac{ \langle \hat{e}_{hj}, \hat{e}_{hi} \rangle_{D_{\mathcal{F}_x^*}^{\alpha\theta}}}{\sin \beta_{ij}} \right) & (i,j) \in E, \\
        -\sum_{k\neq i} w_{ik} & i = j, \\
        0 & \text{else}.
    \end{cases}
\end{equation}
The FLBO at angle $\theta$ is finally given by $\Delta_{\mathit{FLBO}}^{\alpha\theta} = -S_{\mathit{FLBO}}^{-1}W_{\mathit{FLBO}}^{\alpha\theta}$. Note that if we had taken $\omega_\theta \equiv 0$, then we would have recovered the ALBO from \cite{li2020}.
See \cref{fig: overview theory teaser} for a illustrative summary of how we derive the FLBO.

\paragraph{Finsler-based anisotropic convolution operator}
Since our FLBOs are ALBOs, we can use the method of \cite{li2020} to create $\theta$-oriented convolution kernels $g_{\alpha\theta,x}$ according to \cref{eq: anisotropic conv kernel spectral alpha theta} by working in the spectral domain of $\Delta_{\mathit{FLBO}}^{\alpha\theta}$ and taking $\phi_{\alpha\theta,0}, \phi_{\alpha\theta,1} \ldots$ to be the eigenvectors of $\Delta_{\mathit{FLBO}}^{\alpha\theta}$.
Convolution $(f*g)_{\alpha\theta}$ of any function $f$ with this kernel is then computed in the spectral domain 
(\cref{eq: anisotropic conv alpha theta}) 
and then oriented convolutions are summed to get the final convolution operator $(f*g)_\alpha$ (\cref{eq: anisotropic conv alpha}). 
In our experiments, we use eight uniformly spaced samples of $\theta\in [0,\pi)$ as in \cite{li2020}.
Following \cite{li2020}, we encode the spectral coefficients $\hat{g}(\lambda_{\alpha\theta,k})$ of $g_{\alpha\theta,x}$ using Chebychev polynomials
\begin{equation}
    \hat{g}(\lambda_{\alpha\theta,k}) = \sum\limits_{s=0}^{S-1} c_{\alpha\theta,s} T_s(\lambda_{\alpha\theta,k}),
\end{equation}
where $S = 16$ in our tests as in \cite{li2020}, $T_s$ is the Chebychev polynomial of the first kind of order $s$, i.e. $T_s(\cos \gamma) = \cos (s\gamma)$ for any $\gamma$, and $c_{\alpha\theta,s}\in\mathbb{R}$ are coefficients to be learned. 

For visualisation purposes, we provide in the supplementary a plot of convolution filters $g_{\alpha\theta,x}$ for various $\theta$ when choosing the spectral coefficients of the kernel to be equal to one of the Chebychev polynomials, similarly to \cite{li2020}.

\subsection{An Application: Shape matching}

As a concrete example, we use our FLBO for shape matching, which is a common shape analysis application where LBOs play a central role. We compare the results to those based on traditional Riemannian (anisotropic) LBOs, namely FieldConv \cite{Mitchel_2021_ICCV} and ACSCNN \cite{li2020}. Note that \mbox{ACSCNN} coincides with our method when $\tau = 0$. 

\paragraph{Datasets and evaluataion.}
We work on four publicly available datasets: FAUST Remeshed \cite{bogo2014faust,ren2018continuous}, SCAPE Remeshed \cite{anguelov2005scape,ren2018continuous}, and SHREC'16 Partial Cuts \cite{cosmo2016shrec} and Holes \cite{cosmo2016shrec}. We follow the Princeton benchmark protocol \cite{kim2011blended}. 
We provide the percentage of matches that are at most $r$-geodesically distant from the groundtruth correspondence on the reference shape, with $r\in[0,1]$. 
\paragraph{Setting.}
Following \cite{li2020}, our chosen architecture is SplineCNN \cite{fey2018splinecnn}. We train each dataset for $100$ epochs with Adam optimizer and and we set the learning rate to $0.001$. We take as input SHOT descriptors \cite{salti2014shot}. In particular, we use the open-source implementation of ACSCNN\footnote{\url{https://github.com/GCVGroup/ACSCNN}}. The preprocessing step to get the Finsler-based filters is done with our custom implementation.

\subsection{Full mesh correspondence.}

\paragraph{FAUST (Remeshed)} dataset contains 100 human shapes corresponding to 10 different human beings remeshed with the LRVD \cite{yan2014low} to generate more challenging shapes with different mesh connectivity \cite{bogo2014faust, ren2018continuous}. It does not have one-to-one correspondence between shapes and they do not share the same number of vertices.
As suggested in \cite{donati2020deepGeoMaps}, we use the first 80 shapes as the training set and we test on the remaining 20 ones. 

\paragraph{SCAPE (Remeshed)} dataset contains 71 human shapes corresponding to the same human being in different positions. Each shape is remeshed as in FAUST Remeshed \cite{anguelov2005scape, ren2018continuous}. The first 51 shapes form the training set and we test on the other 20 shapes.

\begin{table}
    \centering
    \resizebox{1.0\columnwidth}{!}{
        \begin{tabular}{|l|l|l|l|}
            \hline
            \multicolumn{2}{|l|}{Method} & Accuracy (r = 0) & Accuracy (r = 0.01) \\
            \hline \hline
            \multicolumn{2}{|l|}{FieldConv \cite{Mitchel_2021_ICCV}} & 53.39 & 75.01 \\
            \hline
            \multicolumn{2}{|l|}{ACSCNN \cite{li2020}} & 61.95 & 82.46\\
            \hline
            \multirow{4}{*}{ours} & $\tau = 0.1$ & ${62.61}$   & ${82.90}$  \\ 
            & $\tau = 0.2$ & ${62.58}$ & ${82.90}$ \\ 
            & $\tau = 0.3$ & ${62.67}$  & ${82.94}$  \\ 
            & $\tau = 0.4$ & ${62.27}$ & ${82.66}$  \\
            & $\tau = 0.5$ & ${62.38}$ & ${82.68}$  \\
            \hline
        \end{tabular}
    }
    \caption{Accuracy on FAUST Remeshed.}
    \label{tab: faust remeshed}
\end{table}

\begin{table}
    \centering
    \resizebox{1.0\columnwidth}{!}{
        \begin{tabular}{|l|l|l|l|}
            \hline
            \multicolumn{2}{|l|}{Method} & Accuracy (r = 0) & Accuracy (r = 0.01) \\
            \hline \hline
            \multicolumn{2}{|l|}{FieldConv \cite{Mitchel_2021_ICCV}} & 37.41 & 61.37 \\
            \hline
            \multicolumn{2}{|l|}{ACSCNN \cite{li2020}} & 46.13 & 71.70 \\
            \hline
            \multirow{4}{*}{ours} & $\tau = 0.1$ & ${46.94}$   & ${72.44}$  \\ 
            & $\tau = 0.2$ & ${46.27}$ & ${71.91}$  \\ 
            & $\tau = 0.3$ & 45.91  &  ${71.92}$  \\ 
            & $\tau = 0.4$ & 45.82 &  71.44 \\
            & $\tau = 0.5$ & 45.26 &  71.27 \\
            \hline
        \end{tabular}
    }
    \caption{Accuracy on SCAPE Remeshed.}
    \label{tab: scape remeshed}
\end{table}

\subsection{Partial correspondence.}

\textbf{SHREC'16 Partial Cuts} and \textbf{SHREC'16 Partial Holes} datasets contain respectively 135 and 90 shapes of human and animals divided in 8 categories. The shapes are deformed following two kinds of partiality, respectively cuts and holes. For each category in Cuts, we use the first 10 shapes as a training dataset and we test on the last 5 shapes. For each category in Holes, we use the first 6 shapes as a training dataset and we test on the last 4 shapes.

\subsection{Results}

We present quantitative results in \cref{tab: faust remeshed,tab: scape remeshed,tab: shrec cuts,tab: shrec holes} for FAUST Remeshed, SCAPE Remeshed, SHREC'16 Partial Cuts, and SHREC'16 Partial Holes respectively. For the first two datasets we provide the accuracy within $r\in\{0, 0.01\}$ geodesic radius of the groundtruth, whereas we only display the performance for $r=0$ in the last two datasets due to space constraints as each of these datasets are decomposed into 8 subdatasets. Clearly, FieldConv struggles compared to the more advanced ACSCNN. Our method 
systematically outperforms, or is on par with,
ACSCNN for a default choice of $\tau = 0.1$. This demonstrates that the FLBO can be used in complement to advanced approaches to boost performance, requiring only a small additional preprocessing overhead for computing the metric and the FLBO, which is negligible compared to neural network training times. We test our method with other values of $\tau$ and find that the results remain consistent, proving that our approach is not highly sensitive to finely tuning this hyperparameter. In the literature, providing the results as a plot with respect to all $r\in [0,1]$ is sometimes performed. However, since our results are close to those of ACSCNN, the curves tend to superimpose and do not provide any additional insight, so we do not provide them here.

\begin{table}
    \centering
    \resizebox{\columnwidth}{!}{
        \begin{tabular}{|l|l|l|l|l|l|l|l|l|l|}
            \hline
            \multicolumn{2}{|l|}{Method} & Cat & Centaur & David & Dog & Horse & Michael & Victoria & Wolf \\
            \hline \hline
            \multicolumn{2}{|l|}{FieldConv \cite{Mitchel_2021_ICCV}} & 0.054 & 1.18 & 0 & 0.094 & 1.93 & 0.33 & 0.010 & 45.81\\
            \hline
            \multicolumn{2}{|l|}{ACSCNN \cite{li2020}} & 38.10 & 74.99 & ${23.78}$ & 59.87 & 44.37 & 12.82 & ${31.23}$ & 93.75 \\
            \hline
            \multirow{4}{*}{ours} & $\tau = 0.1$ & ${40.13}$ & ${75.41}$ & 22.34 & 59.31 & 42.72 & ${16.04}$ & 27.95 & ${94.02}$ \\ 
            & $\tau = 0.2$ & $37.72$ & 72.83 & 22.48 & ${60.02}$ & 43.17 & ${14.30}$ & 29.72 & ${94.14}$ \\ 
            & $\tau = 0.3$ & ${40.26}$ & 73.40 & 21.59 & 57.61 & 43.49 & ${13.15}$ & 31.10 & ${94.17}$\\ 
            & $\tau = 0.4$ & ${39.63}$ & 72.29 & 21.51 & 57.12 & ${44.83}$ & ${13.55}$ & 29.35 & 93.33 \\
            & $\tau = 0.5$ & ${40.94}$ & 73.16 & 20.49 & ${60.17}$ & 43.85 & ${15.04}$ & 28.46 & ${94.06}$ \\
            \hline
        \end{tabular}
    }
    \caption{Accuracy on SHREC'16 Partial Cuts.}
    \label{tab: shrec cuts}
\end{table}

\begin{table}[t]
    \centering
    \resizebox{\columnwidth}{!}{
        \begin{tabular}{|l|l|l|l|l|l|l|l|l|l|}
            \hline
            \multicolumn{2}{|l|}{Method} & Cat & Centaur & David & Dog & Horse & Michael & Victoria & Wolf \\
            \hline \hline
            \multicolumn{2}{|l|}{FieldConv \cite{Mitchel_2021_ICCV}} & N.A. & 0.038 & N.A. & 0.11 & N.A. & 0.11 & 0.049 & 0.39 \\
            \hline
            \multicolumn{2}{|l|}{ACSCNN \cite{li2020}} & 19.98 & 32.80 & 16.09 & 41.70 & 46.85 & 12.29 & 11.32 & 84.53\\
            \hline
            \multirow{4}{*}{ours} & $\tau = 0.1$ & 19.95 & ${35.85}$ & ${17.96}$ & ${42.68}$ & ${47.87}$ & 12.00 & ${11.70}$ & ${84.64}$ \\ 
            & $\tau = 0.2$ & 17.99 & ${33.10}$ & ${16.40}$ & ${42.84}$ & ${49.77}$ & ${13.90}$ & 10.02 & ${85.23}$ \\ 
            & $\tau = 0.3$ & ${21.27}$ & ${35.48}$ & ${18.43}$ & ${43.29}$ & 45.07 & 11.31 & ${11.52}$ & ${85.74}$ \\ 
            & $\tau = 0.4$ & ${20.08}$ & 32.54 & ${16.41}$ & ${42.98}$ & 45.71 & ${13.16}$ & ${11.41}$ & ${85.72}$  \\
            & $\tau = 0.5$ & 18.62 & ${35.04}$ & ${16.28}$ & ${42.68}$ & 45.38 & 11.89 & 10.19 & 84.23 \\
            \hline
        \end{tabular}
    }
    \caption{Accuracy on SHREC16' Partial Holes. N.A. indicates a failure to process the dataset.}
    \label{tab: shrec holes}
\end{table}

We plot some qualitative correspondence results in \cref{fig:results shapes 1 per dataset}.  We find that our method is able to provide a convincing mapping between shapes undergoing various highly non-rigid deformations. It is also able to handle missing parts such as cuts and holes, although performance naturally deteriorates as the challenge is significantly harder.

\begin{figure}[ht]%
    \centering
    \includegraphics[width=1.0\columnwidth]{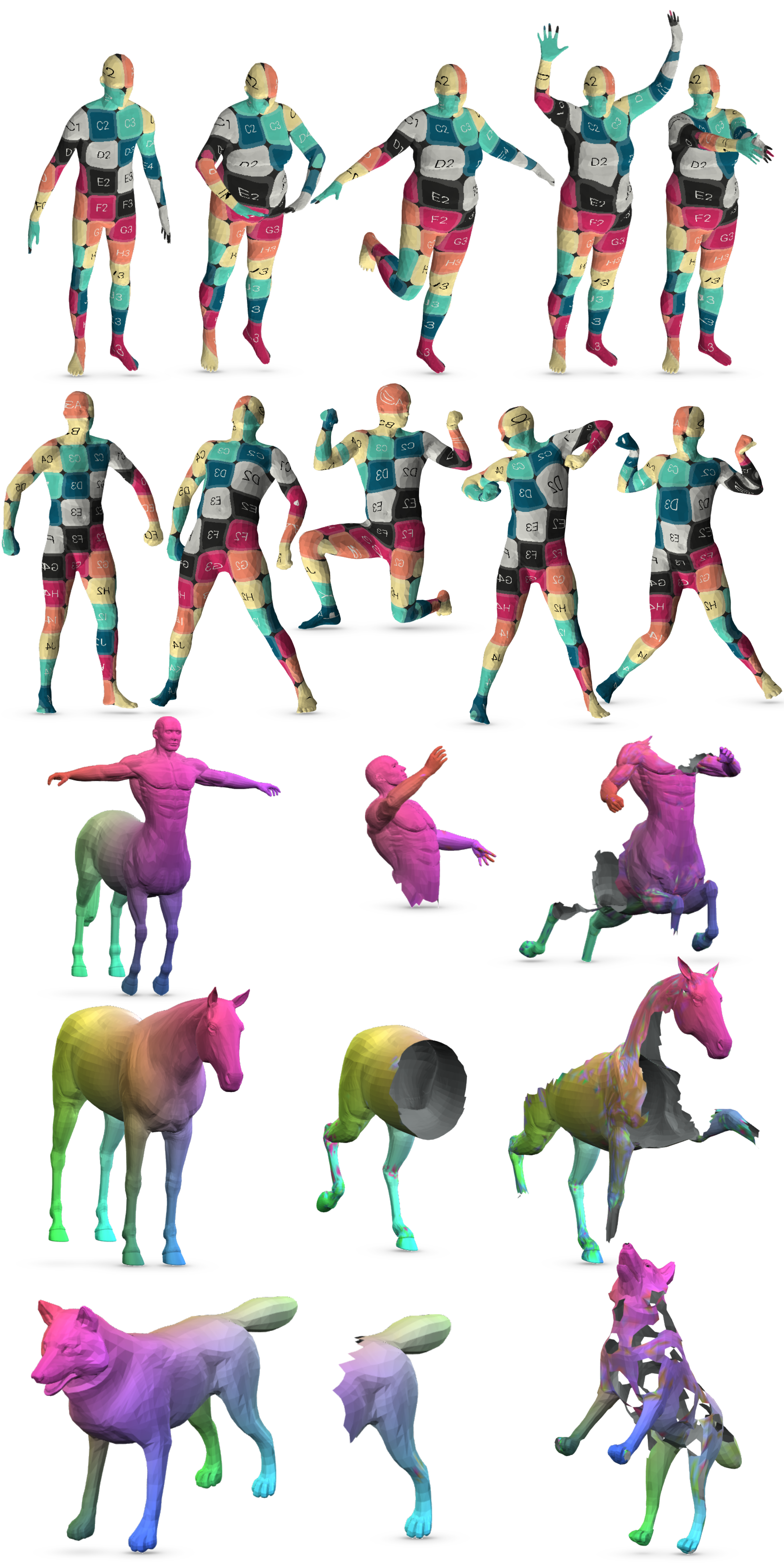}
    \caption{
        Visual shape correspondence results on the FAUST Remeshed (row 1), SCAPE remeshed (row 2), and SHREC'16 Partial Cuts (columns 1 and 2 in rows 3 to 5), and SHREC'16 Partial Holes (columns 1 and 3 in rows 3 to 5) datasets. The source shape is on the left, on which we perform dense correspondence estimation on the shapes to the right. See the supplementary material for more visual results on other shapes.
    }
    \vspace{-1em}
\label{fig:results shapes 1 per dataset}
\end{figure}

%% file: sec_finsler/7_conclusion.tex
\section{Conclusion}

We revisited Finsler manifolds, a generalization of Riemannian manifolds that naturally allow for anisotropies. By exploring their heat diffusion equation we derived a novel Finsler heat kernel and a Finsler-Laplace-Beltrami operator (FLBO). This FLBO generalizes traditional heuristic anisotropic LBOs while being theoretically motivated. It can be used in place of LBOs or ALBOs and is compatible with various advanced techniques such as spatial filtering or geometric deep learning. We tested the pertinence of the FLBO in shape correspondence experiments and found that our approach can easily complement state-of-the-art geometric deep learning methods using ALBOs. We hope that the concise and self-contained review of Finsler geometry and the derivation of a tractable Finsler heat equation and FLBO will encourage further exploration of the Finsler geometry in the field of computer vision.

\vspace*{-0.5em}
\paragraph*{Limitations and future works.}
We have used some strong assumptions to derive the Finsler heat equation. While we acknowledge this limitation, it is necessary, in our work, to exhibit some tractable operators. However, we believe that a further theoretical focus on this stage is an exciting and promising research direction. Our experiments are also built on a straightforward discretization. Although it is not a major drawback, as our goal is to show that traditional LBOs can be generalized to Finsler-based LBOs, an extension of our example to other choices of $(M,\omega)$ would also be crucial. Finally, beyond our simple application to shape matching, we do believe that the importance of Finsler manifolds in computer vision is largely uncharted and requires further stimulating works.

\vspace*{-0.5em}
\paragraph{Acknowledgement.} 
This work was supported by the ERC Advanced Grant SIMULACRON.

%% file: sec_finsler/9_appendix_for_supp_mat.tex
\clearpage

\setcounter{page}{1}
\setcounter{section}{0}

\maketitlesupplementary

\noindent This supplementary material is organized as follows:\\
\noindent \textbf{\cref{sec:proof}} provides some proofs concerning properties of Randers metrics mentioned in the main paper: positivity, dual formula, and bounded norm of the dual Randers vector field. \\
\noindent \textbf{\cref{sec:prop1}} details the proof of \cref{prop: heat eqn simplified randers distance trick}. \\
\noindent \textbf{\cref{sec:runtime}} gives additional details about runtime in our experiments.\\
\noindent \textbf{\cref{sec:figures}} shows additional visual results of anisotropic kernels and shape correspondence.\\

\section{Classical proofs for the Randers metric}\label{sec:proof}

The proofs in this section of the supplementary material are not novel in the Randers metric community. The purpose of putting them here is to present a self-contained theory that is fully and simply explained, avoiding the need for readers to search through other challenging mathematical papers. For prior existing proofs, see for example \cite{mirebeau2014}. Note that errors exist in various reference materials, which further justifies rederiving the proofs. For instance, in \cite{mirebeau2014} the formulae for the dual Randers metric in Prop.~5.1 are incorrect.

\subsection{Proof of Randers metric positivity}

Any vector in the tangent space $T_xX$ can be given in the form $M(x)^{-1}v$. We can then write
\begin{align}
    &\mathcal{F}_x(M(x)^{-1}v) = \sqrt{v^\top M(x)^{-1} v} + \omega(x)^\top  M(x)^{-1}v\\
    &\quad= \lVert M(x)^{-\tfrac{1}{2}}v\rVert + \langle M(x)^{-\tfrac{1}{2}}\omega(x),  M(x)^{-\tfrac{1}{2}}v\rangle.
\end{align}

Using the Cauchy-Schwartz inequality, we get $\lvert\omega(x)^\top M(x)^{-1}v\rvert \le \lVert M(x)^{-\tfrac{1}{2}}\omega(x) \rVert \lVert M(x)^{-\tfrac{1}{2}} v\rVert$. The assumption on $\omega$ provides $\lVert M(x)^{-\tfrac{1}{2}}\omega(x)\rVert <1$, thus as soon as $v\neq 0$, we get $\mathcal{F}_x(M(x)^{-1}v) > 0$. 
\qed

\subsection{Proof of dual Randers metric formulae}

For conciseness, we drop the explicit dependence on $x$. The dual Finsler metric $\mathcal{F}^*$ of the Randers metric $\mathcal{F}$ is given by \cref{def: dual finsler}. By positive homogeneity of $\mathcal{F}$, it is also given by
\begin{equation}
    \mathcal{F}^*(v) = \max \left\{\frac{\langle v, b \rangle}{\mathcal{F}(b)} ; b\neq 0 \right\}.
\end{equation}
As such, $\tfrac{1}{\mathcal{F}^*(v)} = \min \left\{\tfrac{\mathcal{F}(b)}{\langle v, b \rangle} ; b\neq 0 \right\}$, and once again by homogeneity we have
\begin{equation}
    \label{eq: 1 / F* = min F}
    \frac{1}{\mathcal{F}^*(v)} = \min \{\mathcal{F}(b) ; \langle v, b\rangle = 1 \}.
\end{equation}
The Lagrangian for this constrained optimization problem is given by $L(b, \mu) =  \mathcal{F}(b) + \mu \langle v, b \rangle$. Computing the gradient $\tfrac{\partial L}{\partial b}$ and setting it to $0$ to satisfy the KKT conditions, the optimal $b$, denoted $b^*$, satisfies
\begin{equation}
    \label{eq: randers dL/db = 0}
    M\frac{b^*}{\lVert b^* \rVert_M} + \omega + \mu v = 0.
\end{equation}
On one hand, by computing the scalar product with $b^*$, and recalling that the constraint enforces $\langle v, b^*\rangle = 1$, and recalling the formula for the Randers metric (\cref{eq: randers}) we get $\mu = -\mathcal{F}(b^*)$. Since $b^*$ solves the optimization problem in \cref{eq: 1 / F* = min F}, we have
\begin{equation}
    \label{eq: randers mu = - 1/F*}
    \mu = -\frac{1}{\mathcal{F}^*(v)}.
\end{equation}
On the other hand, we can compute $\lVert \omega + \mu v \rVert_{M^{-1}}$ from \cref{eq: randers dL/db = 0}
\begin{equation}
    \lVert \omega + \mu v \rVert_{M^{-1}} = \frac{\lVert Mb^* \rVert_{M^{-1}}}{\lVert b^* \rVert_M} = 1.
\end{equation}
Taking the square, we get a degree two polynomial for which $\mu$ is a root
\begin{equation}
    \mu^2 \lVert v \rVert_{M^{-1}}^2  + 2\mu \langle \omega, v\rangle_{M^{-1}} + \lVert \omega\rVert_{M^{-1}}^2 - 1 = 0.
\end{equation}
Writing $\alpha = 1 - \lVert \omega \rVert_{M^{-1}}^2 >0$, the roots are given by
\begin{equation}
    \mu_\pm = \frac{-\langle \omega, v\rangle_{M^{-1}} \pm \sqrt{
\langle \omega, v\rangle_{M^{-1}}^2 + \lVert v\rVert_{M^{-1}}^2 \alpha}}{\lVert v \rVert_{M^{-1}}^2}.
\end{equation}
Clearly, $\mu_+ > 0$ and $\mu_- < 0$ for $v\neq 0$, yet $\mu < 0$ following $\mu = -\mathcal{F}(b^*)$ as the metric is always positive. Therefore $\mu = \mu_-$, and then by inverting \cref{eq: randers mu = - 1/F*} we get
\begin{align}
    \mathcal{F}^*(v) &= \frac{\lVert v\rVert_{M^{-1}}^2}{\langle \omega, v\rangle_{M^{-1}} + \sqrt{\langle \omega, v\rangle_{M^{-1}}^2 + \lVert v\rVert_{M^{-1}}^2 \alpha}} \\
    &= \sqrt{v^\top \frac{1}{\alpha^2}\left( \alpha M^{-1} + M^{-1}\omega \omega^\top M^{-1}\right) v} \nonumber\\
    &\hspace{3em}-\frac{1}{\alpha} \langle \omega, v\rangle_{M^{-1}},
\end{align}
where we used the classical trick $\tfrac{1}{x+\sqrt{y}} = \tfrac{x-\sqrt{y}}{x^2-y}$ to remove the square root from the denominator.  We now recognise the dual metric $\mathcal{F}^*$ of the Randers metric $\mathcal{F}$ as a Randers metric associated to $(M^*, \omega^*)$ with
\begin{equation}
    \begin{cases}
        M^* = \frac{1}{\alpha^2}(\alpha M^{-1} + (M^{-1}\omega) (M^{-1}\omega)^\top ), \\
        \omega^* = -\frac{1}{\alpha}M^{-1}\omega.
    \end{cases}
    \qed
\end{equation}

\subsection{Proof of $\lVert \omega^*\rVert_{{M*}^{-1}(x)} < 1$}

For conciseness, we drop the explicit dependence on $x$ and we denote $\tilde\omega = M^{-\tfrac{1}{2}}\omega$. Recalling that $M$ is symmetric and that $\alpha>0$, and plugging $M^{-1}\omega = M^{-\tfrac{1}{2}}\tilde\omega$ into \cref{eq: M* formula,eq: omega* formula}, we get
\begin{equation}
    \lVert \omega^*\rVert_{{M^*}^{-1}}^2
    = \tilde{\omega}^\top
    (\alpha I +  \tilde\omega\tilde\omega^\top)^{-1}
    \tilde\omega.
\end{equation}

We then invoke the Sherman-Morrison formula
$\left( A+uv^\top\right)^{-1} = A^{-1} - \frac{1}{1+ v^\top A^{-1} u}A^{-1}uv^\top A^{-1}$ for matrix $A = I$ and vectors $u = v = \alpha^{-\frac{1}{2}}\tilde\omega$. Writing $q_{M,\omega} = \alpha^{-1}\tilde\omega^\top \tilde\omega = \alpha^{-1}\omega^\top M^{-1}\omega$ for conciseness, we get
\begin{align}
    \lVert \omega^*\rVert_{{M^*(x)}^{-1}}^2
    &= q_{M, \omega}\left(1  - 
    \frac{q_{M,\omega}}{1 + q_{M,\omega}}  \right) \\
    &=\frac{q_{M,\omega}}{1+q_{M,\omega}}\\
    &= 1 - \frac{1}{1 + \alpha^{-1}\omega^\top M^{-1} \omega}<1. 
    \qed
\end{align}

\begin{figure}[tb]%
  \centerline{%
  \begin{tabular}{cc}%
    \includegraphics[height=6.2cm]{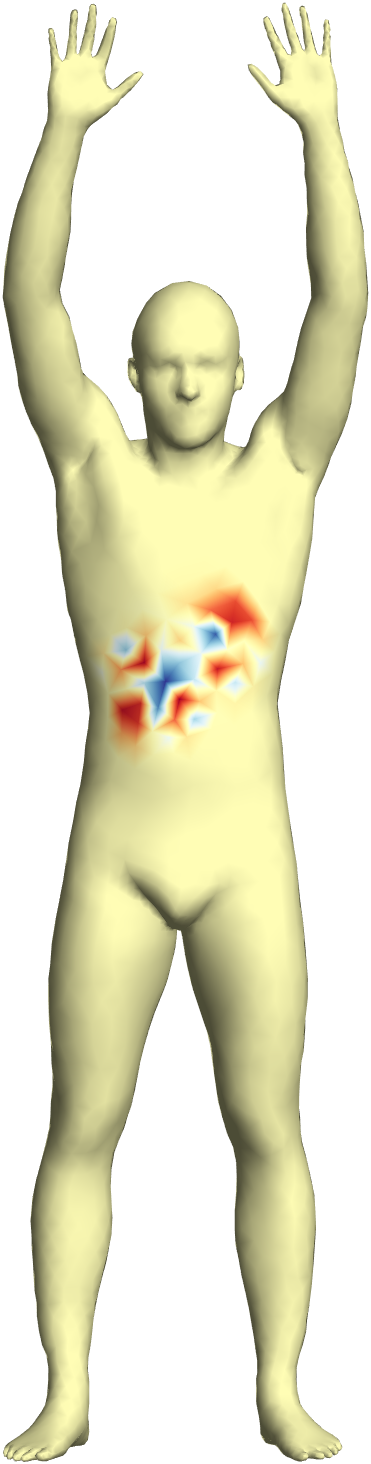}
    &
    \hspace{1cm}
    \includegraphics[height=6.2cm]{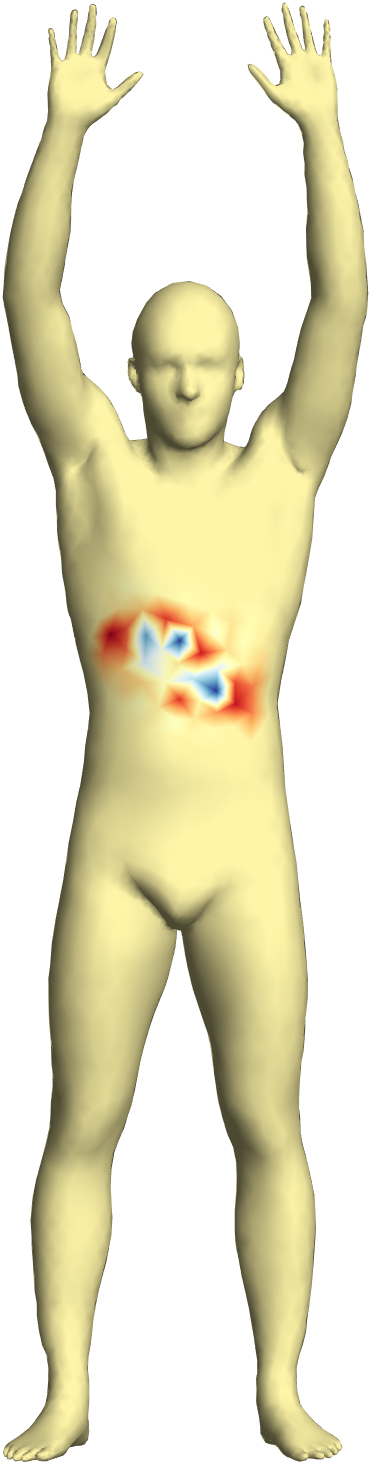}
  \end{tabular}
  }
\captionof{figure}{
Finsler-based anisotropic kernels $g_{\alpha\theta,x}$ centered at the same point $x$ with different rotation angles $\theta$ equal to $0$ (left) and $\tfrac{3\pi}{8}$ (right). Here the filter is chosen to be the Chebychev polynomial $\hat{g}(\lambda) = T_5(\lambda)$.}
\label{fig:filter}
\end{figure}

\section{Proof of \cref{prop: heat eqn simplified randers distance trick}}\label{sec:prop1}

For conciseness, we drop the explicit dependence on $x$. 
Denoting $\mathcal{P}(u) = \mathcal{F}^*(\nabla u)\nabla \mathcal{F}^*(\nabla u)$, the Finsler heat equation is given by
\begin{equation}
    \partial_t u = \diverg(\mathcal{P}(u)).
\end{equation}
Differentiating the Finsler dual metric \cref{eq: F* randers formula}, we have
\begin{equation}
    \label{eq: grad F*}
    \nabla \mathcal{F}^*(\nabla u) = M^* \frac{\nabla u}{\lVert\nabla u\rVert_{M^*}} + \omega^*.
\end{equation}
Given the initial condition, the solution $u$ will approximately satisfy $\mathcal{F}^*(\nabla u) = 1$ in short time, i.e. for small time steps $t$.
As such, in short time we have
\begin{equation}
    \label{eq: P(u) with F*=1}
    \mathcal{P}(u) = M^* \frac{\nabla u}{\lVert\nabla u\rVert_{M^*}} + \omega^*.
\end{equation}
Following \cite{bonnans2022linear} (Corrolary 2.10), if $u$ follows this condition, then there exists a positive scalar $\lambda$ giving the proportionality between the following vectors
\begin{equation}
    \label{eq: g1 lambda g2}
    M^* \frac{\nabla u}{\lVert\nabla u\rVert_{M^*}} + \omega^* = 
    \lambda \big((M^* - \omega^*{\omega^*}^\top)\nabla u + \omega^*\big).
\end{equation}
This is due to the fact that $\nabla u$ belongs to the shared 0-level set of functions of $v$ on the tangent space $f_1(v) = \lVert v\rVert_{M^*} + {\omega^*}^\top v - 1$ and $f_2(v) = \lVert v\rVert_{M^* - \omega^*{\omega^*}^\top}^2 + 2{\omega^*}^\top v - 1$ which share the same sign for all $v$, and with gradients computed at $v=\nabla u$ equal to $\nabla f_1(\nabla u) = M^* \frac{\nabla u}{\lVert\nabla u\rVert_{M^*}} + \omega^*$ and $\nabla f_2 (\nabla u) = 2 ((M^* - \omega^*{\omega^*}^\top)\nabla u + \omega^*)$. By computing the scalar product of \cref{eq: g1 lambda g2}, i.e. of the equation $\nabla f_1(\nabla u) = \tfrac{\lambda}{2}\nabla f_2(\nabla u)$, with $\nabla u$, and recalling that $\mathcal{F}^*(\nabla u) = 1$, we have that
\begin{equation}
    \nabla u^\top \nabla f_1(\nabla u)  = \mathcal{F}^*(\nabla u) = 1
\end{equation}
and that
\begin{align}
    \nabla u^\top \frac{1}{2}\nabla f_2(\nabla u) &= \lVert \nabla u\rVert_{M^* - \omega^*{\omega^*}^\top}^2 + (1-\lVert \nabla u \rVert_{M^*}) \\
    &= \lVert \nabla u\rVert_{M^*}^2 - (1-\lVert \nabla u\rVert_{M^*})^2 \nonumber\\
    &\hspace{2em}+ (1-\lVert \nabla u\rVert_{M^*})\\
    &= \lVert \nabla u\rVert_{M^*}.
\end{align}
Another calculation gives us that
\begin{align}
    1 + \lVert \nabla u\rVert_{M^* - \omega^*{\omega^*}^\top}^2 &= 1 + \lVert \nabla u\rVert_{M^*}^2 - (1 - \lVert \nabla u\rVert_{M^*})^2\\
    &= 2\lVert \nabla u\rVert_{M^*},
\end{align}
thus $\nabla u^\top \frac{1}{2}\nabla f_2(\nabla u) = \tfrac{1}{2}(1 + \lVert \nabla u\rVert_{M^* - \omega^*{\omega^*}^\top}^2)$. As $\nabla f_1(\nabla u) = \tfrac{\lambda}{2}\nabla f_2(\nabla u)$, we can deduce that 
\begin{equation}
    \lambda = \frac{2}{1 + \nabla u^\top (M^* - \omega^*{\omega^*}^\top)\nabla u}.
\end{equation}
Under the assumption that the metric is close to Riemannian, i.e. $\lVert \omega\rVert$ is small, then we can approximate $\lVert \nabla u\rVert_{M^*} = \mathcal{F}^*(\nabla u)= 1$, and thus approximate the proportionality coefficient $\lambda$ by
\begin{equation}
    \lambda = \frac{2}{1+\lVert \nabla u\rVert_{M^*}} = 1.
\end{equation}
Using \cref{eq: P(u) with F*=1,eq: g1 lambda g2}, we conclude that for this particular initialisation the heat equation becomes in short time
\begin{equation}
    \partial_t u = \diverg\left((M^* -\omega^*{\omega^*}^\top)\nabla u \right) + \diverg(\omega^*). \qed
\end{equation}

\section{Implementation and runtime}\label{sec:runtime}

In terms of runtime and memory, the difference with \cite{li2020} comes from a small overhead due to computing the metric as preprocessing. Our additional actions involve cheap basic operations on $3\times 3$ matrices, e.g. inversion (\cref{eq: M = H^-1}) and multiplication (\cref{eq: M* formula,eq: omega* formula,eq: H = R U D U^T R^T,eq: omega = tau R u}). Note that further preprocessing operations, e.g.~\cref{eq: weights w flbo theta}, are the same as those in \cite{li2020} and thus have the same cost. For example, our implementation on a CPU `2 GHz Intel Core i5' takes on average 1.70s per shape in the FAUST (Remeshed) dataset (1.82s per shape in the SCAPE (Remeshed) dataset), compared to 1.55s (resp.~1.76s) for \cite{li2020}. Once this preprocessing is performed, we obtain an anisotropic LBO discretised as a matrix (line 432) of the same size as in \cite{li2020}. Therefore, runtime of the core shape-matching learning algorithm is the same and is significantly longer, around 30s per epoch for FAUST (Remeshed) (resp.~10s for SCAPE (Remeshed)) on 1 GPU `Quadro P6000'.

\section{Further visual results}\label{sec:figures}

\subsection{Example kernels}

We show in \cref{fig:filter} two of our anisotropic kernels $g_{\alpha\theta,x}$ at different angles, when choosing a predetermined Chebychev polynomial for its spectral decomposition.

\subsection{More shape correspondence results}

We display further visual shape correspondence results in \cref{fig:qual-faustrm,fig:qual-scaperm,fig:qual-shrec-cuts,fig:qual-shrec-holes}. Our method clearly outperforms FieldConv while providing on par visual results to ACSCNN. The superiority of our method is revealed in the quantitative results of the main paper.

\begin{figure*}%
  \centerline{%
  \begin{tabular}{c|c|ccccc}%
    &
    \multirow{2}{*}[-1cm]{
    \rotatebox[origin=c]{90}{FieldConv}}
    &
    \raisebox{-0.5\height}{
    \includegraphics[height=3.4cm]{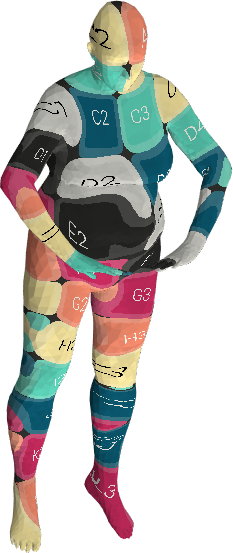} 
    }
    &
    \raisebox{-0.5\height}{
    \includegraphics[height=3.4cm]{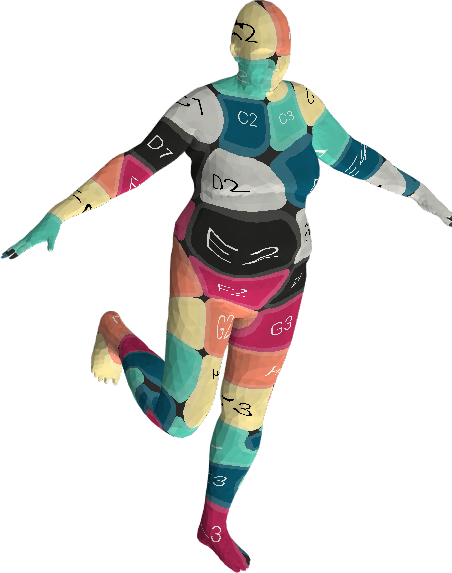} 
    }
    &
    \raisebox{-0.5\height}{
    \includegraphics[height=3.4cm]{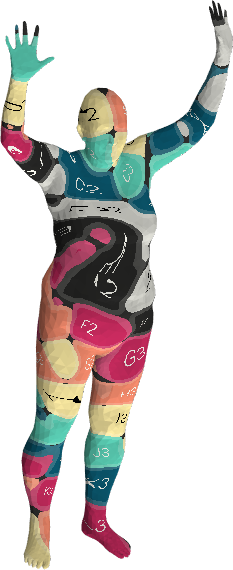} 
    }
    &
    \raisebox{-0.5\height}{
    \includegraphics[height=3.4cm]{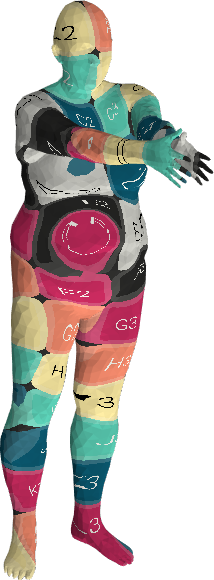} 
    }
    &
    \raisebox{-0.5\height}{
    \includegraphics[height=3.4cm]{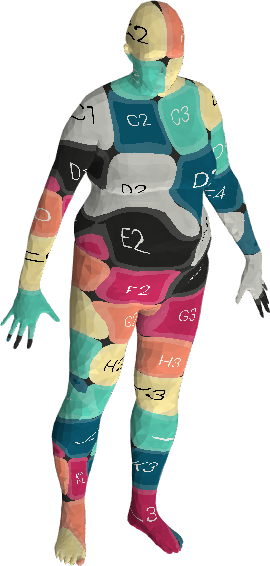} 
    }
    \\
    &
    &
    \raisebox{-0.5\height}{
    \includegraphics[height=3.4cm]{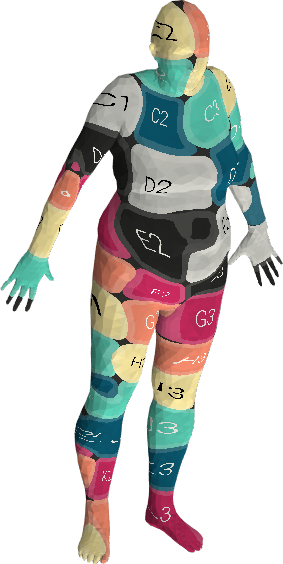} 
    }
    &
    \raisebox{-0.5\height}{
    \includegraphics[height=3.4cm]{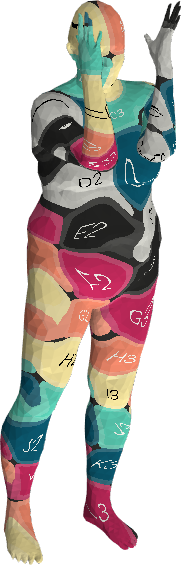} 
    }
    &
    \raisebox{-0.5\height}{
    \includegraphics[height=3.4cm]{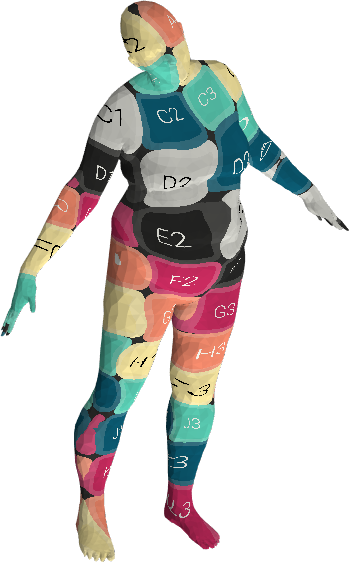} 
    }
    &
    \raisebox{-0.5\height}{
    \includegraphics[height=3.4cm]{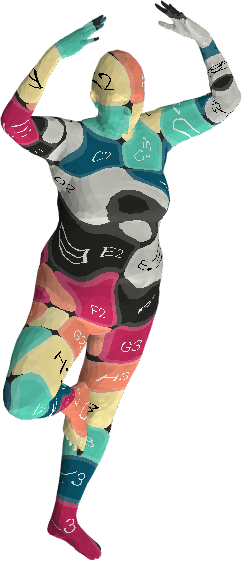} 
    }
    &
    \raisebox{-0.5\height}{
    \includegraphics[height=3.4cm]{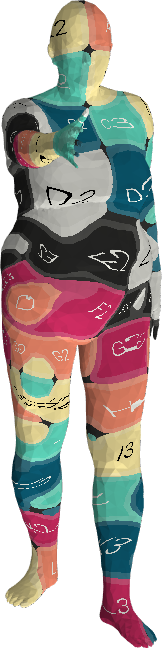} 
    }
    \\
    &&&&&&
    \\
    \cline{2-7}&&&&&&
    \\
    \multirow{2}{*}[0cm]{
    \raisebox{-0.5\height}{
    \includegraphics[height=3.4cm]{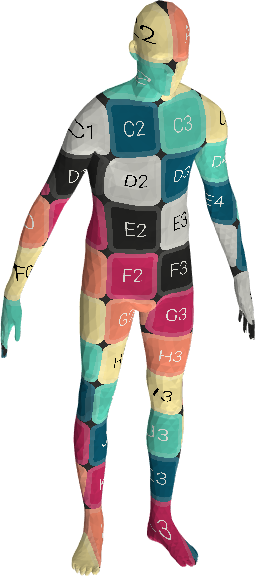}
    }
    }
    &
    \multirow{2}{*}[-1cm]{
    \rotatebox[origin=c]{90}{ACSCNN}}
    &
    \raisebox{-0.5\height}{
    \includegraphics[height=3.4cm]{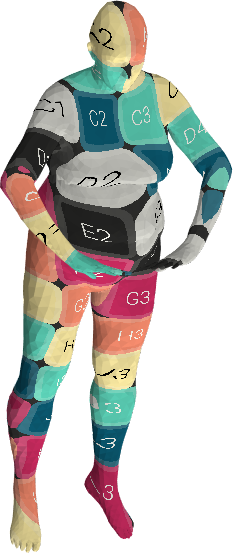} 
    }
    &
    \raisebox{-0.5\height}{
    \includegraphics[height=3.4cm]{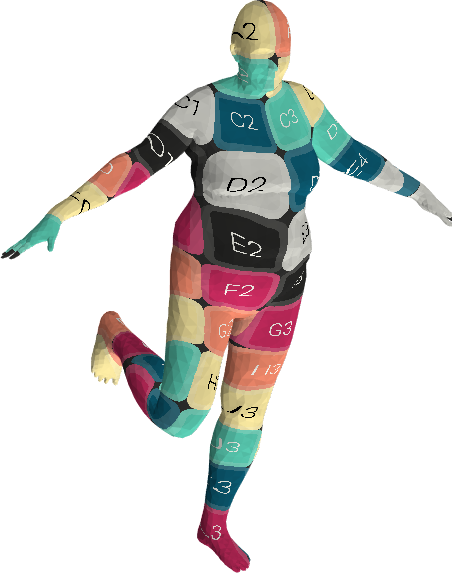} 
    }
    &
    \raisebox{-0.5\height}{
    \includegraphics[height=3.4cm]{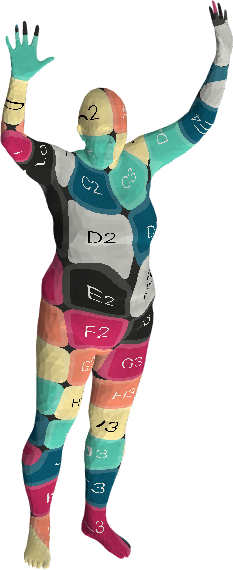} 
    }
    &
    \raisebox{-0.5\height}{
    \includegraphics[height=3.4cm]{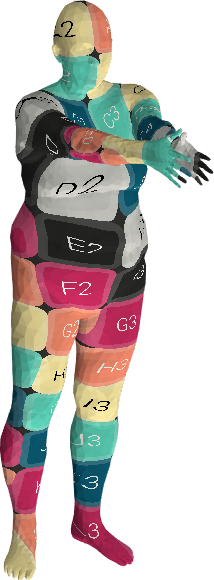} 
    }
    &
    \raisebox{-0.5\height}{
    \includegraphics[height=3.4cm]{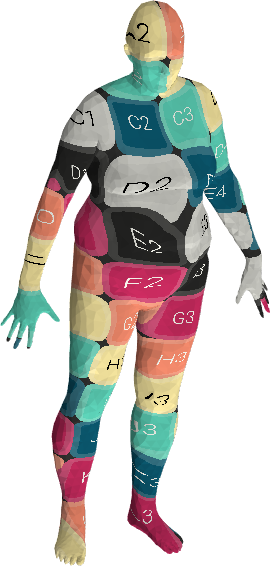} 
    }
    \\
    &
    &
    \raisebox{-0.5\height}{
    \includegraphics[height=3.4cm]{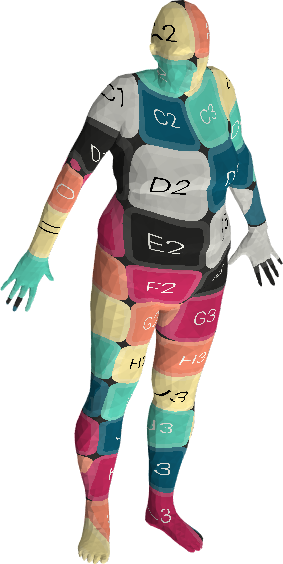} 
    }
    &
    \raisebox{-0.5\height}{
    \includegraphics[height=3.4cm]{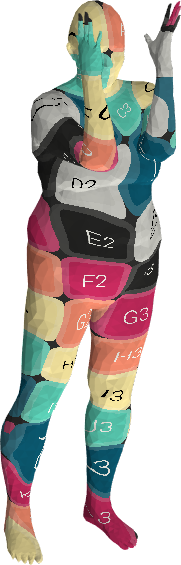} 
    }
    &
    \raisebox{-0.5\height}{
    \includegraphics[height=3.4cm]{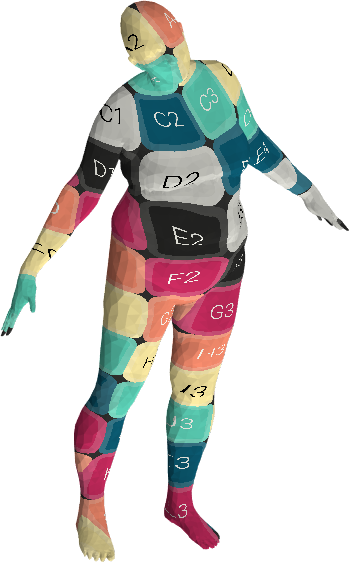} 
    }
    &
    \raisebox{-0.5\height}{
    \includegraphics[height=3.4cm]{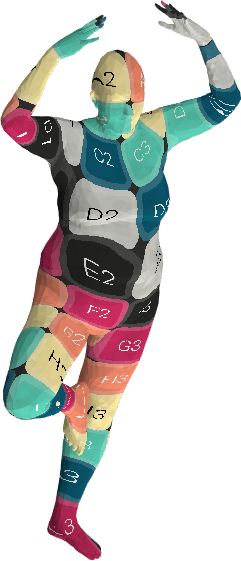} 
    }
    &
    \raisebox{-0.5\height}{
    \includegraphics[height=3.4cm]{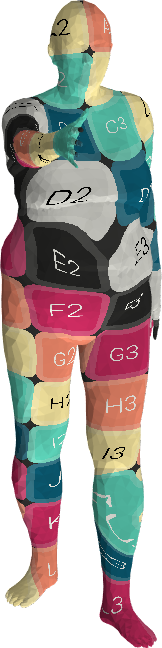} 
    }
    \\
    &&&&&&
    \\
    \cline{2-7}&&&&&&
    \\
    &
        \multirow{2}{*}[-1.3cm]{
    \rotatebox[origin=c]{90}{Ours}}
    &
    \raisebox{-0.5\height}{
    \includegraphics[height=3.4cm]{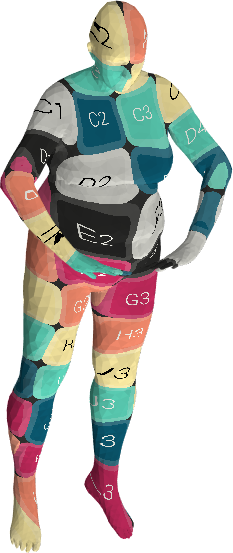}
    }
    &
    \raisebox{-0.5\height}{
    \includegraphics[height=3.4cm]{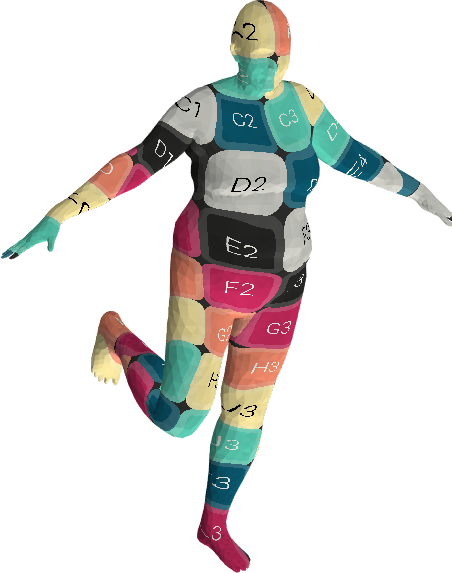}
    }
    &
    \raisebox{-0.5\height}{
    \includegraphics[height=3.4cm]{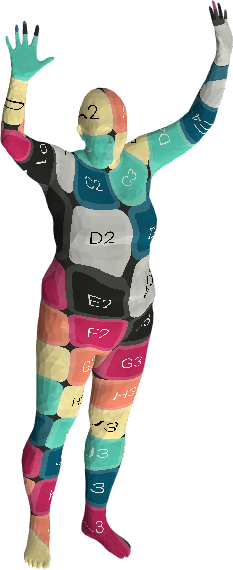}
    }
    &
    \raisebox{-0.5\height}{
    \includegraphics[height=3.4cm]{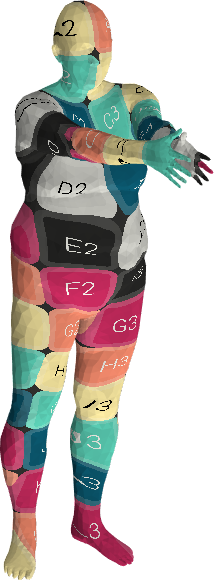} 
    }
    &
    \raisebox{-0.5\height}{
    \includegraphics[height=3.4cm]{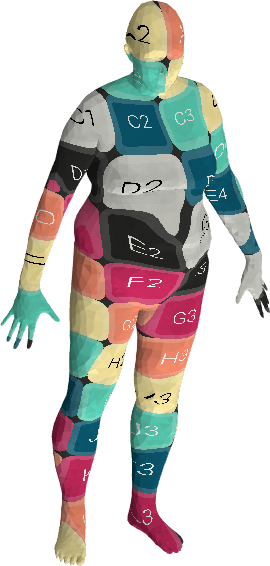} 
    }
    \\
    &
    &
    \raisebox{-0.5\height}{
    \includegraphics[height=3.4cm]{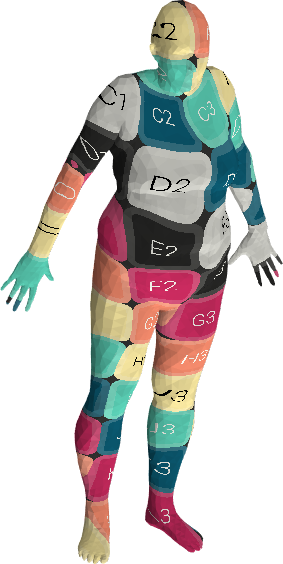}
    }
    &
    \raisebox{-0.5\height}{
    \includegraphics[height=3.4cm]{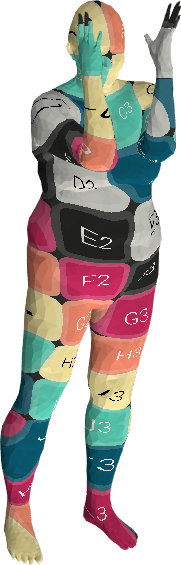}
    }
    &
    \raisebox{-0.5\height}{
    \includegraphics[height=3.4cm]{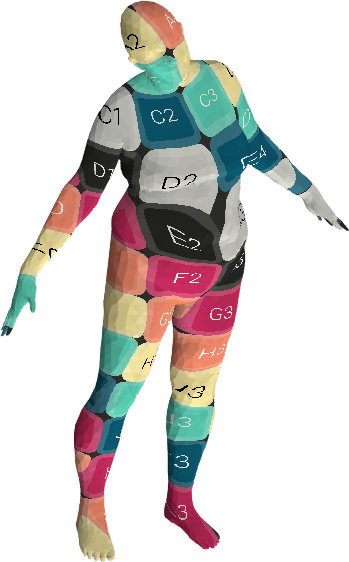}
    }
    &
    \raisebox{-0.5\height}{
    \includegraphics[height=3.4cm]{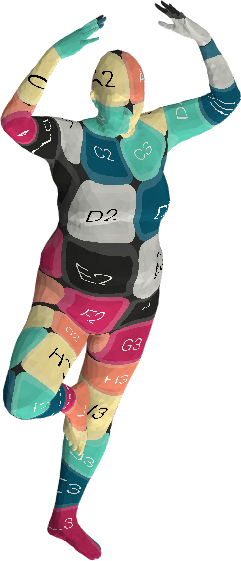} 
    }
    &
    \raisebox{-0.5\height}{
    \includegraphics[height=3.4cm]{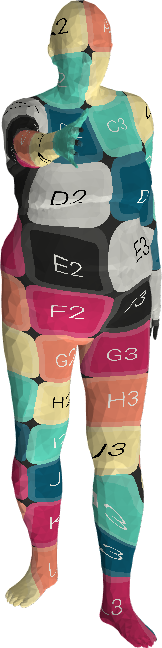} 
    }
  \end{tabular}
  }
\captionof{figure}{
Visual shape correspondence results on the FAUST Remeshed dataset. The source shape is on the left, on which we perform dense correspondence estimation on the shapes to the right, using either FieldConv, ACSCNN, or our approach.
}
\label{fig:qual-faustrm}
\end{figure*}

\begin{figure*}%
  \centerline{%
  \begin{tabular}{c|c|ccccc}%
    &
    \multirow{2}{*}[-1cm]{
    \rotatebox[origin=c]{90}{FieldConv}}
    &
    \raisebox{-0.5\height}{
    \includegraphics[height = 3.4cm]{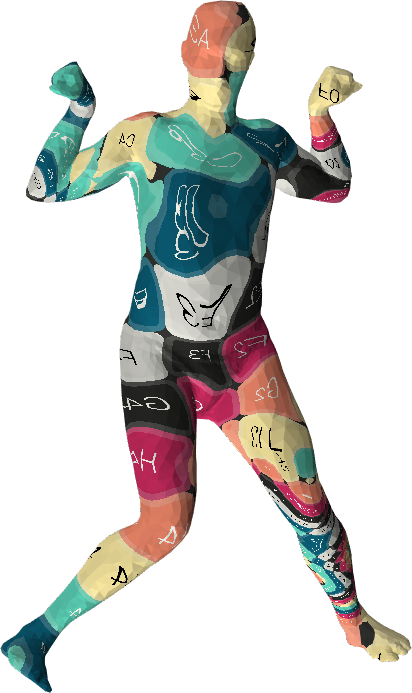} 
    }
    &
    \raisebox{-0.5\height}{
    \includegraphics[height = 3.4cm]{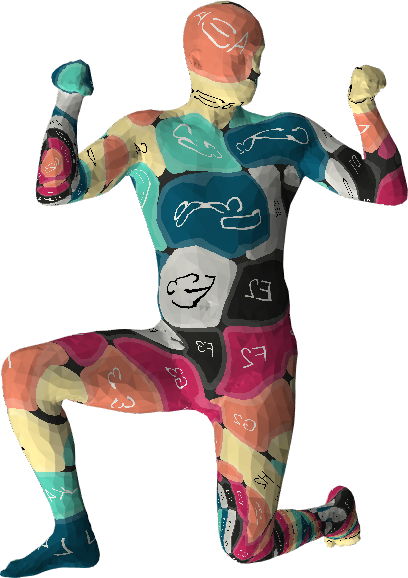} 
    }
    &
    \raisebox{-0.5\height}{
    \includegraphics[height = 3.4cm]{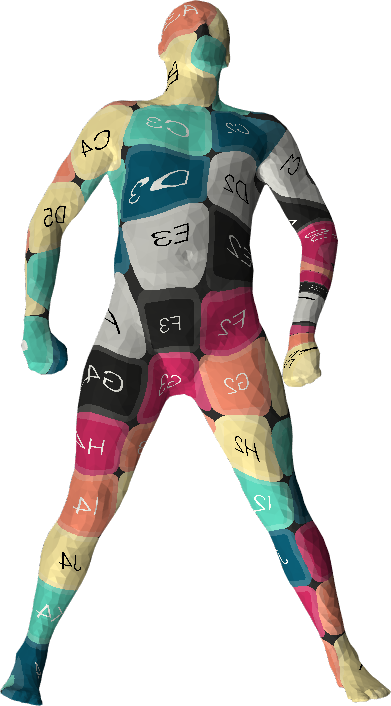} 
    }
    &
    \raisebox{-0.5\height}{
    \includegraphics[height = 3.4cm]{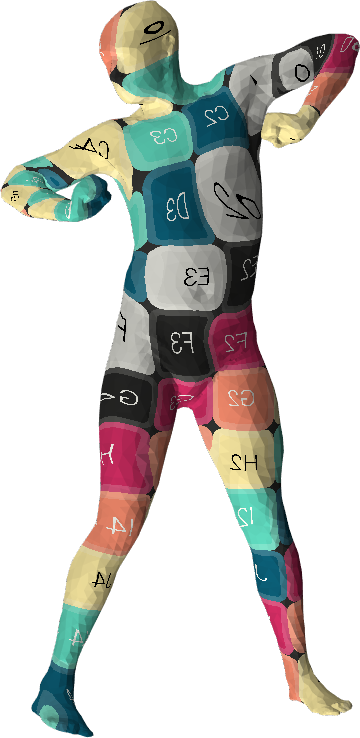} 
    }
    &
    \raisebox{-0.5\height}{
    \includegraphics[height = 3.4cm]{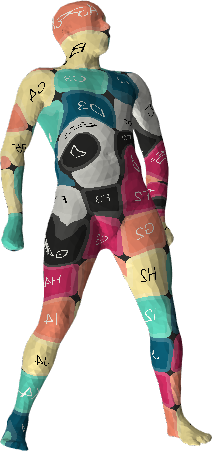} 
    }    
    \\
    &
    &
    \raisebox{-0.5\height}{
    \includegraphics[height = 3.4cm]{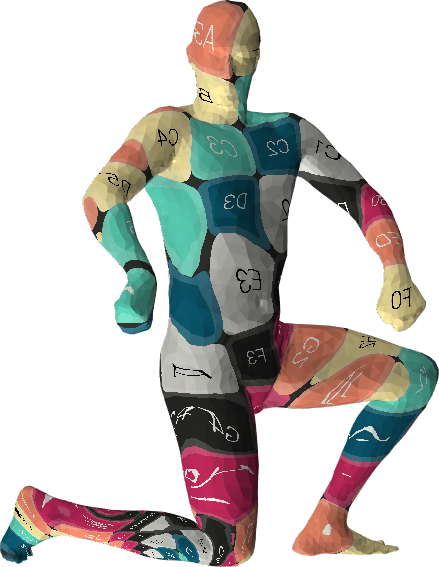} 
    }
    &
    \raisebox{-0.5\height}{
    \includegraphics[height = 3.4cm]{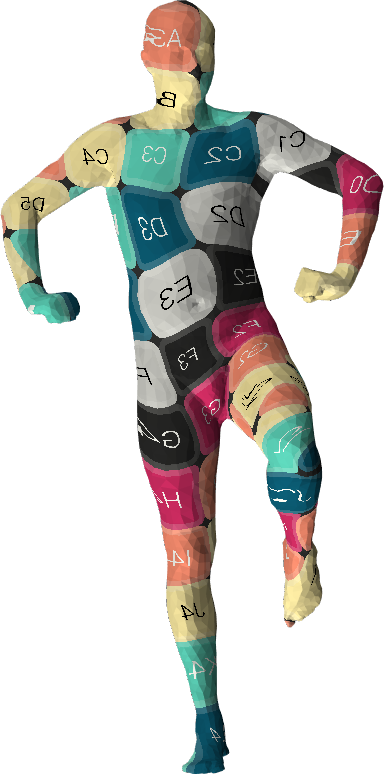} 
    }
    &
    \raisebox{-0.5\height}{
    \includegraphics[height = 3.4cm]{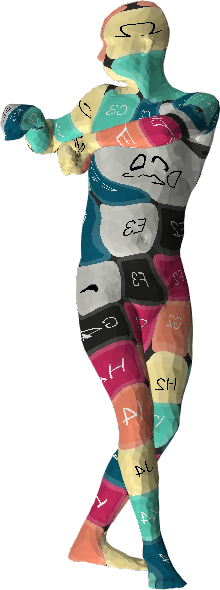} 
    }
    &
    \raisebox{-0.5\height}{
    \includegraphics[height = 3.4cm]{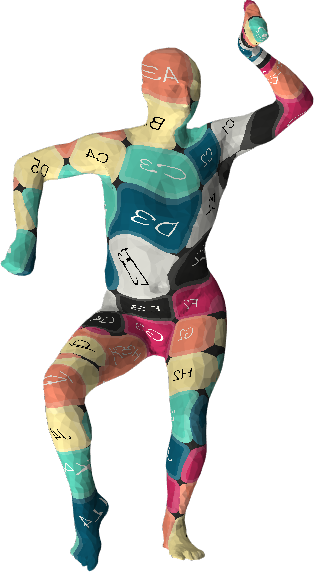} 
    }
    &
    \raisebox{-0.5\height}{
    \includegraphics[height = 3.4cm]{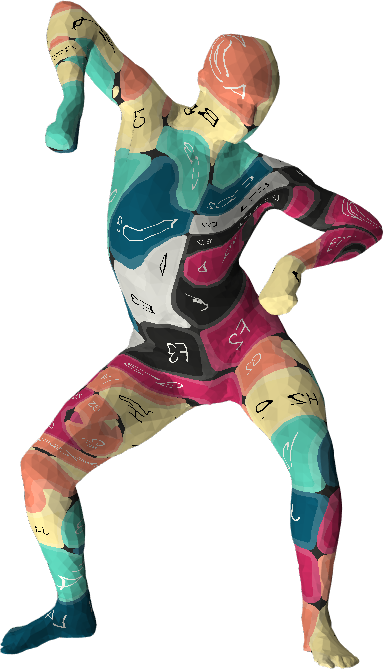} 
    }
    \\
    &&&&&&
    \\
    \cline{2-7}&&&&&&
    \\
    \multirow{2}{*}[0cm]{
    \raisebox{-0.5\height}{
    \includegraphics[height = 3.4cm]{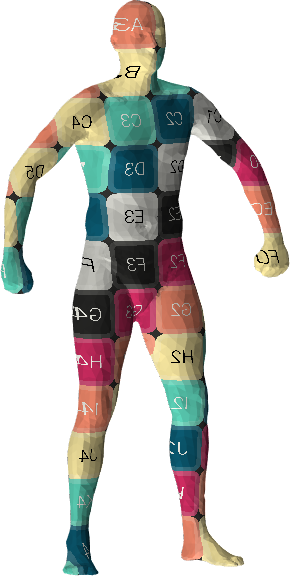}
    }
    }
    &
    \multirow{2}{*}[-1cm]{
    \rotatebox[origin=c]{90}{ACSCNN}}
    &
    \raisebox{-0.5\height}{
    \includegraphics[height = 3.4cm]{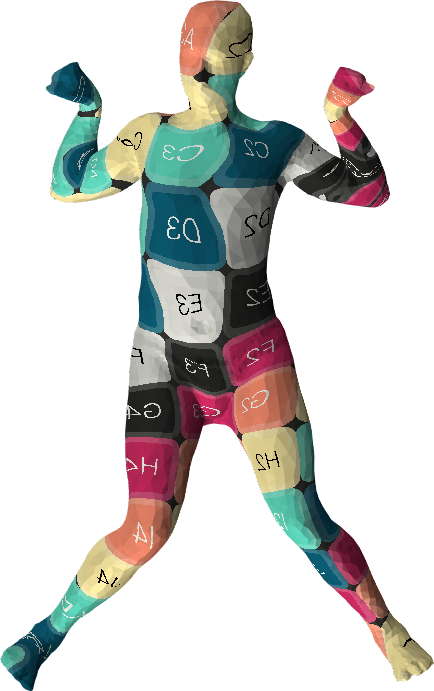} 
    }
    &
    \raisebox{-0.5\height}{
    \includegraphics[height = 3.4cm]{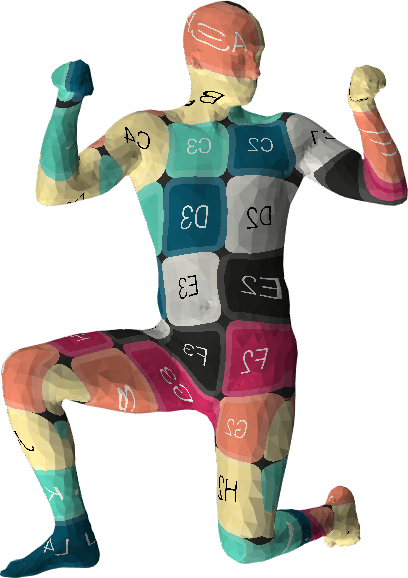} 
    }
    &
    \raisebox{-0.5\height}{
    \includegraphics[height = 3.4cm]{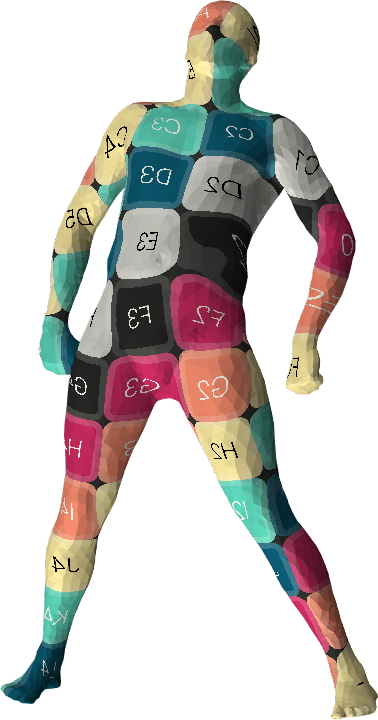} 
    }
    &
    \raisebox{-0.5\height}{
    \includegraphics[height = 3.4cm]{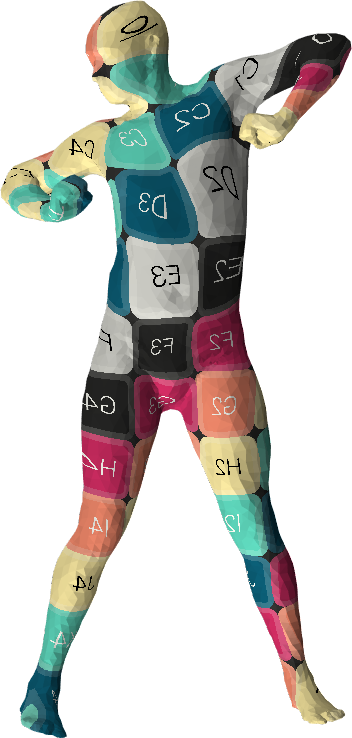} 
    }
    &
    \raisebox{-0.5\height}{
    \includegraphics[height = 3.4cm]{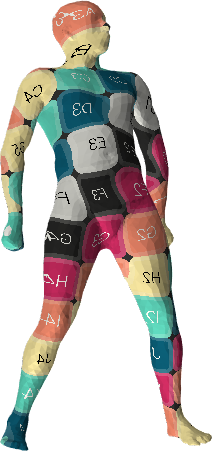} 
    }
    \\
    &
    &
    \raisebox{-0.5\height}{
    \includegraphics[height = 3.4cm]{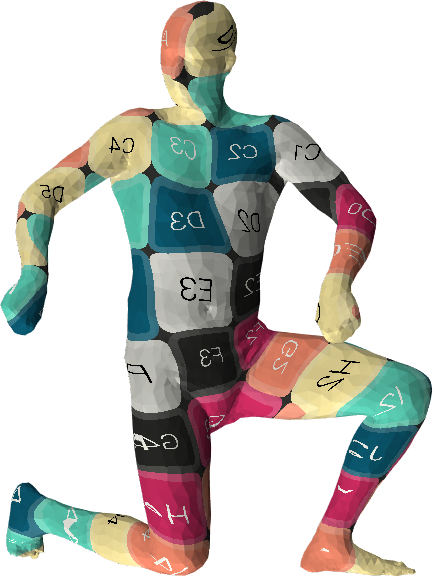} 
    }
    &
    \raisebox{-0.5\height}{
    \includegraphics[height = 3.4cm]{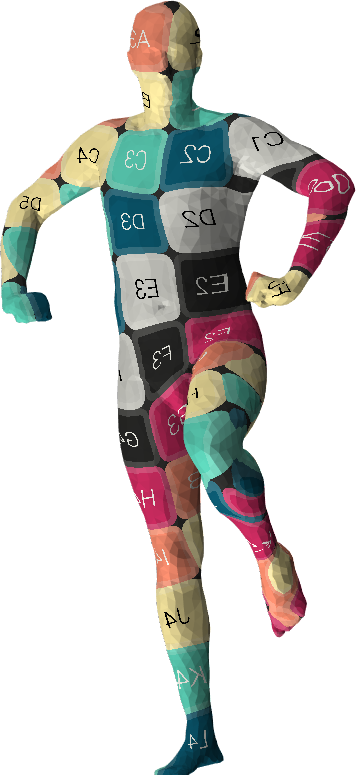} 
    }
    &
    \raisebox{-0.5\height}{
    \includegraphics[height = 3.4cm]{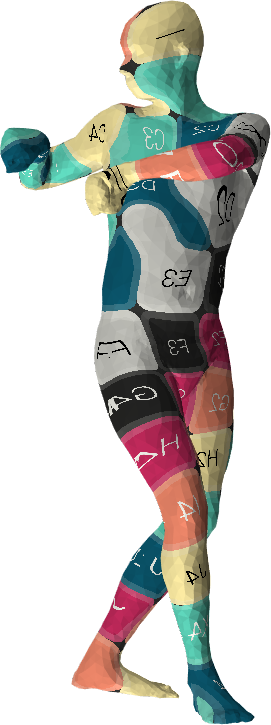} 
    }
    &
    \raisebox{-0.5\height}{
    \includegraphics[height = 3.4cm]{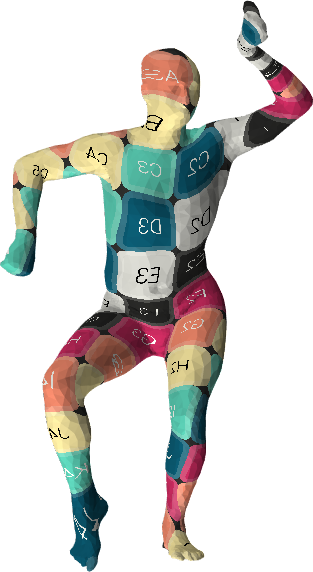} 
    }
    &
    \raisebox{-0.5\height}{
    \includegraphics[height = 3.4cm]{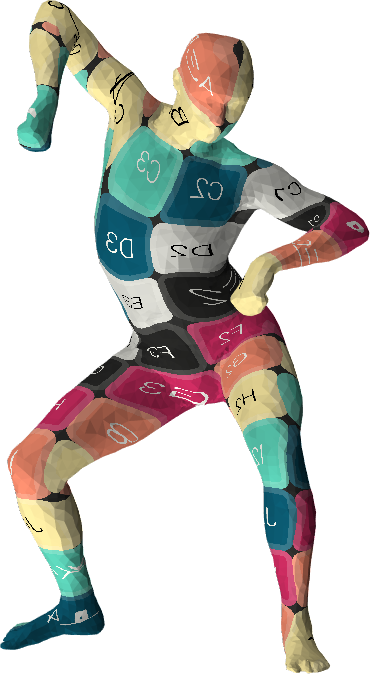} 
    }
    \\
    &&&&&&
    \\
    \cline{2-7}&&&&&&
    \\
    &
    \multirow{2}{*}[-1.5cm]{
    \rotatebox[origin=c]{90}{Ours}}
    &
    \raisebox{-0.5\height}{
    \includegraphics[height = 3.4cm]{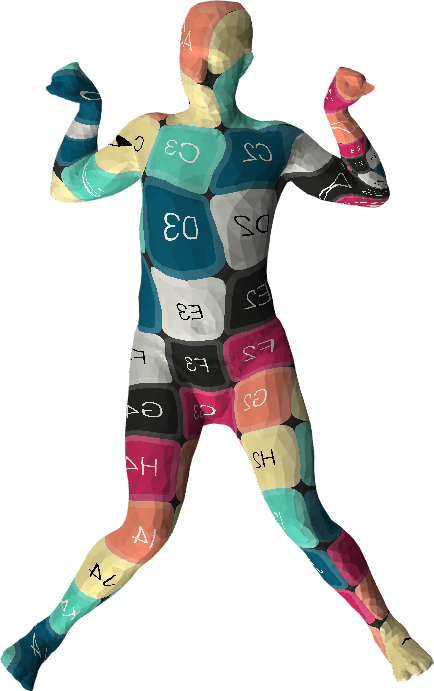}
    }
    &
    \raisebox{-0.5\height}{
    \includegraphics[height = 3.4cm]{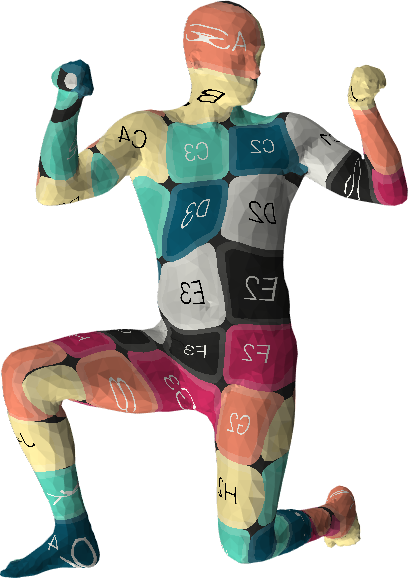}
    }
    &
    \raisebox{-0.5\height}{
    \includegraphics[height = 3.4cm]{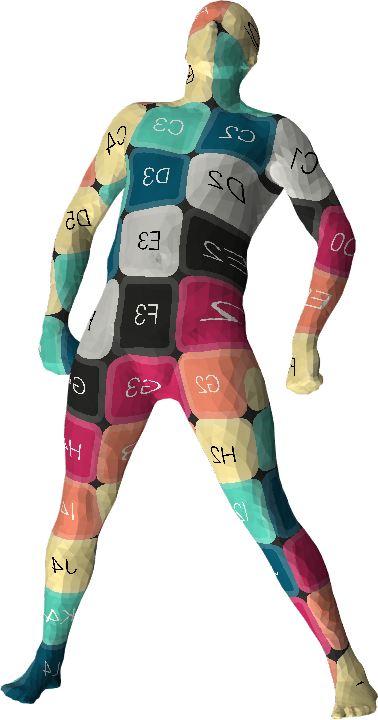}
    }
    &
    \raisebox{-0.5\height}{
    \includegraphics[height = 3.4cm]{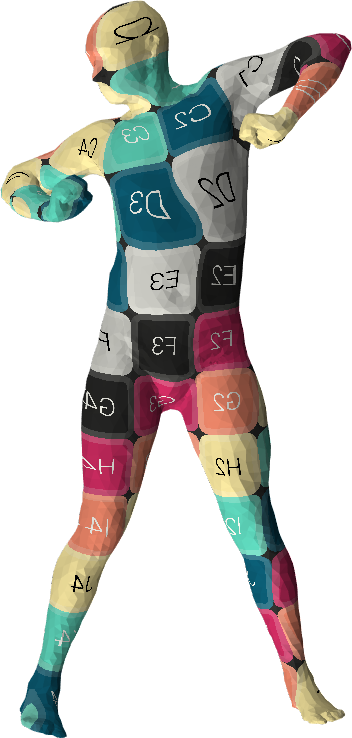} 
    }
    &
    \raisebox{-0.5\height}{
    \includegraphics[height = 3.4cm]{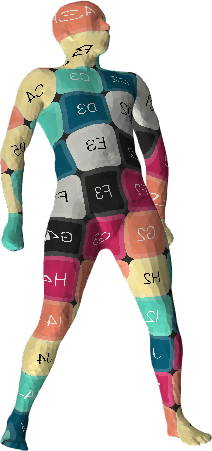} 
    }    
    \\
    &
    &
    \raisebox{-0.5\height}{
    \includegraphics[height = 3.4cm]{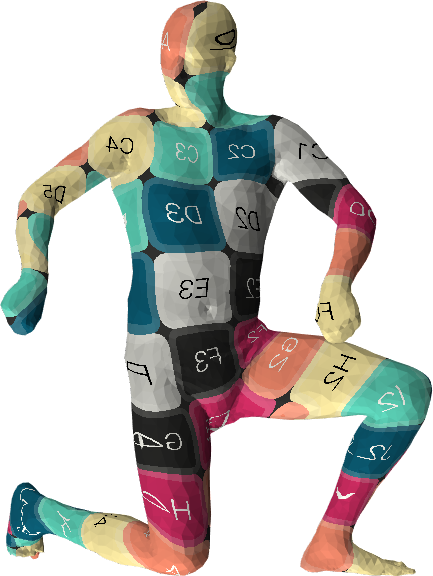}
    }
    &
    \raisebox{-0.5\height}{
    \includegraphics[height = 3.4cm]{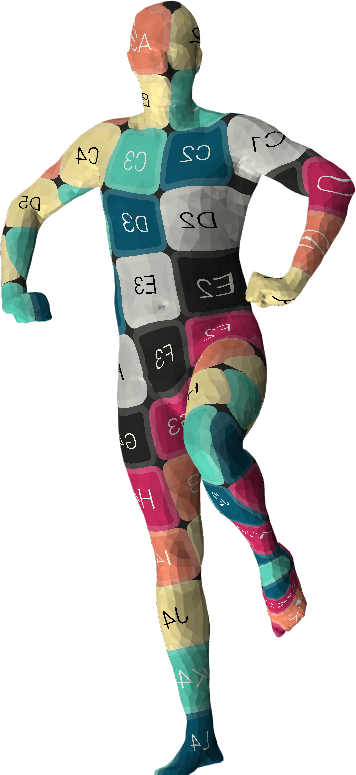}
    }
    &
    \raisebox{-0.5\height}{
    \includegraphics[height = 3.4cm]{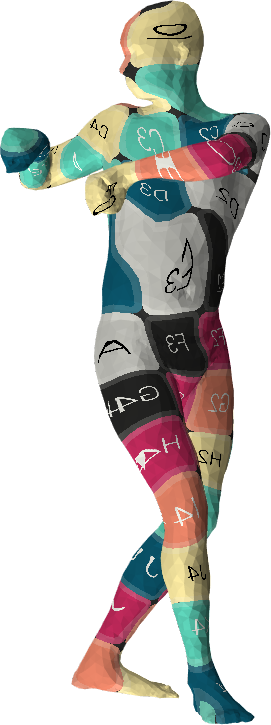}
    }
    &
    \raisebox{-0.5\height}{
    \includegraphics[height = 3.4cm]{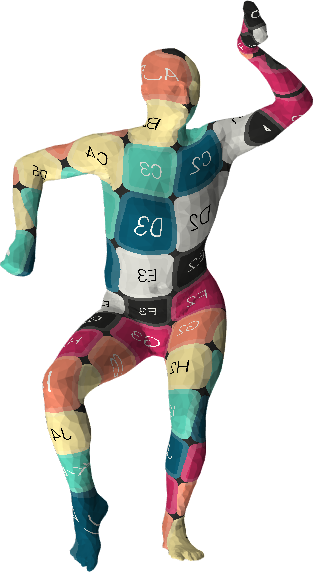} 
    }
    &
    \raisebox{-0.5\height}{
    \includegraphics[height = 3.4cm]{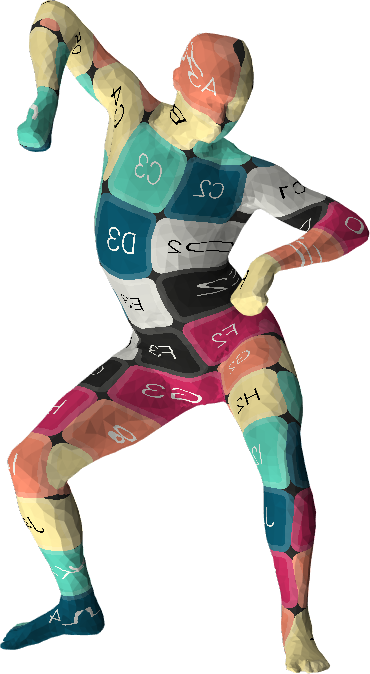} 
    }
  \end{tabular}
  }
\captionof{figure}{
Visual shape correspondence results on the SCAPE Remeshed dataset. The source shape is on the left, on which we perform dense correspondence estimation on the shapes to the right, using either FieldConv, ACSCNN, or our approach.}
\label{fig:qual-scaperm}
\end{figure*}


\begin{figure*}%
  \centerline{%
  \begin{tabular}{c|ccc|ccc}%
    & \multicolumn{3}{|c|}{ACSCNN} & \multicolumn{3}{|c}{Ours}
    \\
    &&&&&&
    \\
    \cline{2-7} &&&&&&
    \\
    \raisebox{-0.5\height}{
    \includegraphics[width = 0.06\paperwidth]{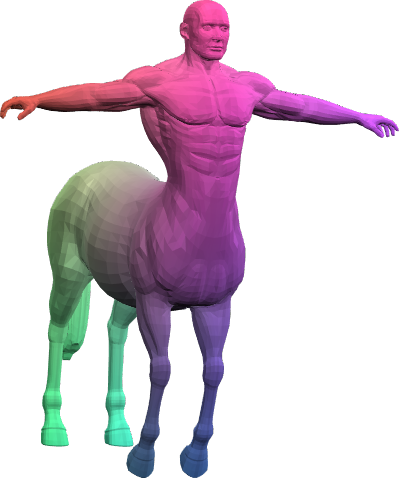} 
    }
    &
    \raisebox{-0.5\height}{
    \includegraphics[width = 0.06\paperwidth]{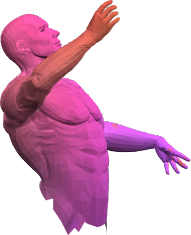} 
    }
    &
    \raisebox{-0.5\height}{
    \includegraphics[width = 0.06\paperwidth]{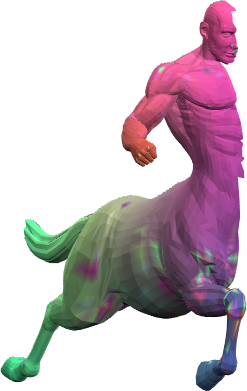} 
    }
    &
    \raisebox{-0.5\height}{
    \includegraphics[width = 0.06\paperwidth]{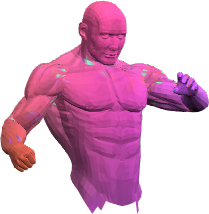} 
    }    
        &
    \raisebox{-0.5\height}{
    \includegraphics[width = 0.06\paperwidth]{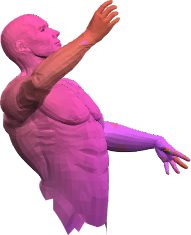} 
    }
    &
    \raisebox{-0.5\height}{
    \includegraphics[width = 0.06\paperwidth]{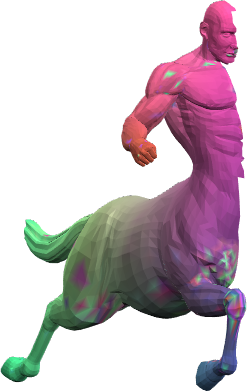} 
    }
    &
    \raisebox{-0.5\height}{
    \includegraphics[width = 0.06\paperwidth]{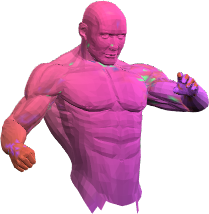} 
    } 
    \\
    \raisebox{-0.5\height}{
    \includegraphics[width = 0.06\paperwidth]{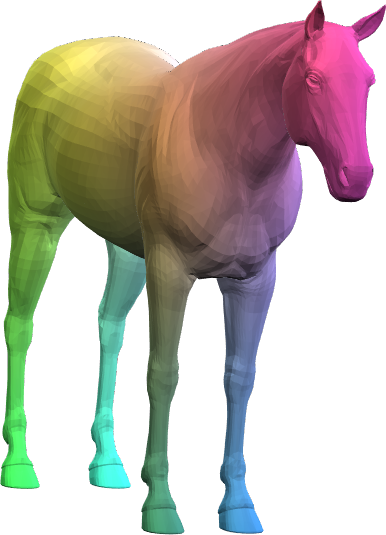} 
    }
    &
    \raisebox{-0.5\height}{
    \includegraphics[width = 0.06\paperwidth]{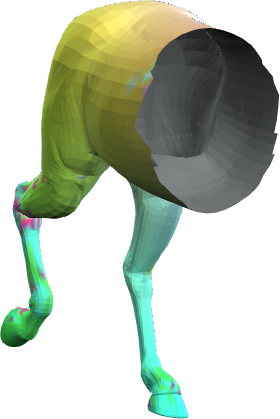} 
    }
    &
    \raisebox{-0.5\height}{
    \includegraphics[width = 0.06\paperwidth]{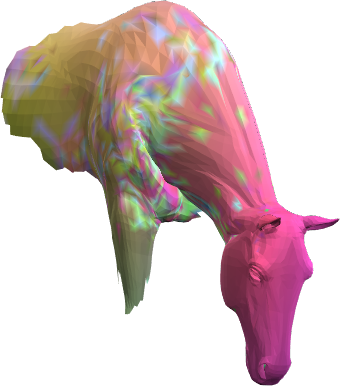} 
    }
    &
    \raisebox{-0.5\height}{
    \includegraphics[width = 0.06\paperwidth]{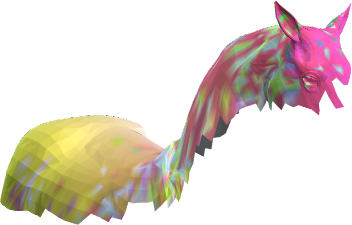} 
    }    
        &
    \raisebox{-0.5\height}{
    \includegraphics[width = 0.06\paperwidth]{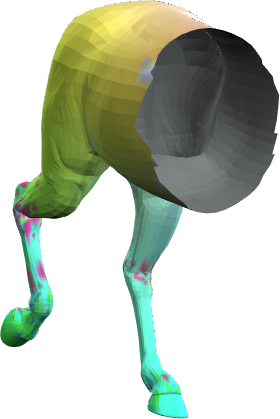} 
    }
    &
    \raisebox{-0.5\height}{
    \includegraphics[width = 0.06\paperwidth]{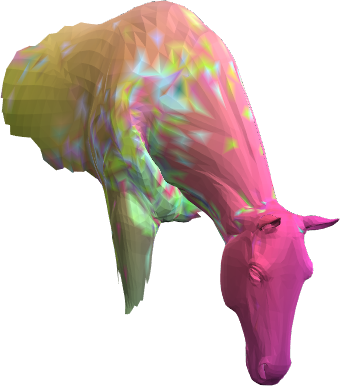} 
    }
    &
    \raisebox{-0.5\height}{
    \includegraphics[width = 0.06\paperwidth]{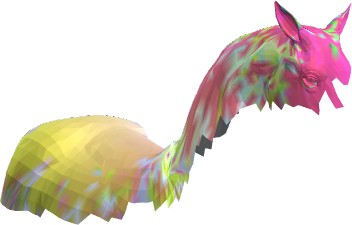} 
    } 
    \\
    \raisebox{-0.5\height}{
    \includegraphics[width = 0.06\paperwidth]{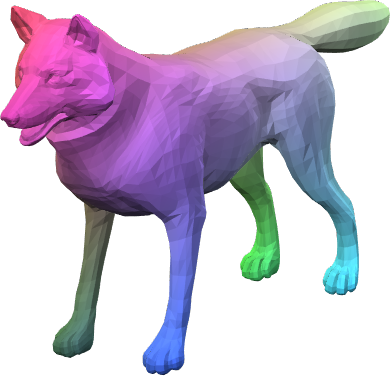} 
    }
    &
    \raisebox{-0.5\height}{
    \includegraphics[width = 0.06\paperwidth]{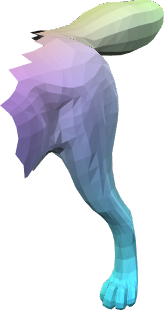} 
    }
    &
    \raisebox{-0.5\height}{
    \includegraphics[width = 0.06\paperwidth]{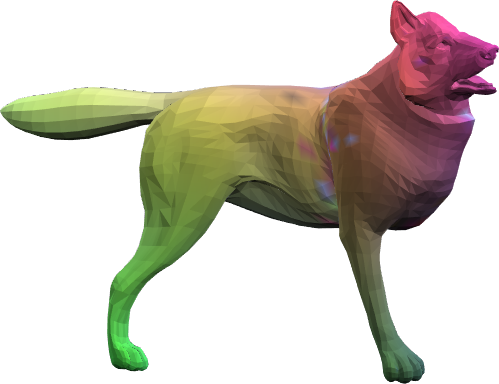} 
    }
    &
    \raisebox{-0.5\height}{
    \includegraphics[width = 0.06\paperwidth]{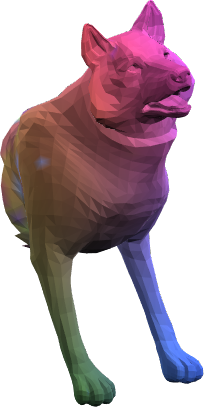} 
    }    
        &
    \raisebox{-0.5\height}{
    \includegraphics[width = 0.06\paperwidth]{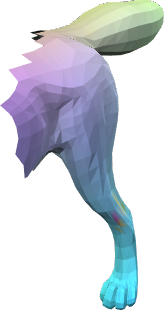} 
    }
    &
    \raisebox{-0.5\height}{
    \includegraphics[width = 0.06\paperwidth]{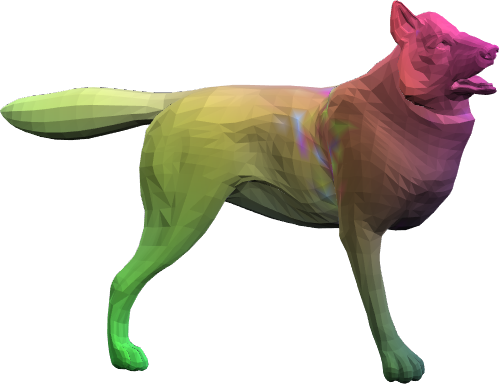} 
    }
    &
    \raisebox{-0.5\height}{
    \includegraphics[width = 0.06\paperwidth]{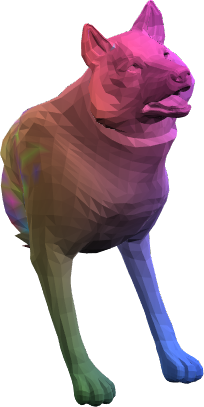} 
    } 
    \\
    \raisebox{-0.5\height}{
    \includegraphics[width = 0.06\paperwidth]{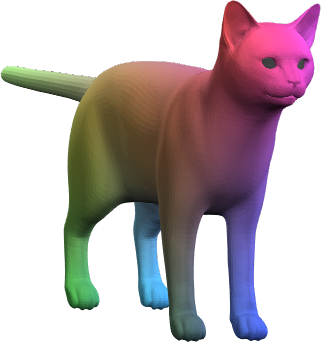} 
    }
    &
    \raisebox{-0.5\height}{
    \includegraphics[width = 0.06\paperwidth]{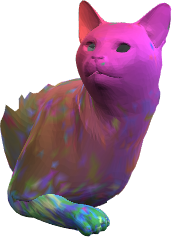} 
    }
    &
    \raisebox{-0.5\height}{
    \includegraphics[width = 0.06\paperwidth]{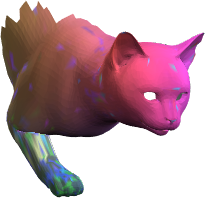} 
    }
    &
    \raisebox{-0.5\height}{
    \includegraphics[width = 0.06\paperwidth]{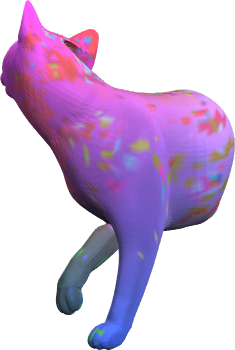} 
    }    
        &
    \raisebox{-0.5\height}{
    \includegraphics[width = 0.06\paperwidth]{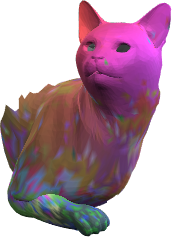} 
    }
    &
    \raisebox{-0.5\height}{
    \includegraphics[width = 0.06\paperwidth]{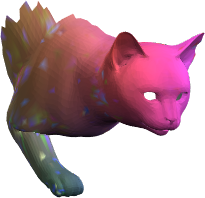} 
    }
    &
    \raisebox{-0.5\height}{
    \includegraphics[width = 0.06\paperwidth]{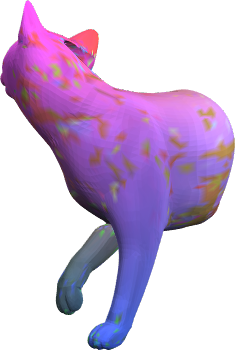} 
    } 
  \end{tabular}
  }
\captionof{figure}{
Visual shape correspondence results on the SHREC'16 Cuts dataset. The source shape is on the left, on which we perform dense correspondence estimation on the shapes to the right, using either ACSCNN or our approach. FieldConv fails to provide meaningful correspondence results on partial shapes so we do not display its qualitative performance.}
\label{fig:qual-shrec-cuts}
\end{figure*}

\begin{figure*}%
  \centerline{%
  \begin{tabular}{c|ccc|ccc}%
    & \multicolumn{3}{|c|}{ACSCNN} & \multicolumn{3}{|c}{Ours}
    \\
    &&&&&&
    \\
    \cline{2-7} &&&&&&
    \\
    \raisebox{-0.5\height}{
    \includegraphics[width = 0.06\paperwidth]{Pictures/render/shrec16holes/color02/holes_centaur_shape_7_M.png} 
    }
    &
    \raisebox{-0.5\height}{
    \includegraphics[width = 0.06\paperwidth]{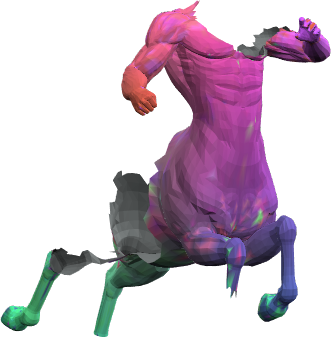} 
    }
    &
    \raisebox{-0.5\height}{
    \includegraphics[width = 0.06\paperwidth]{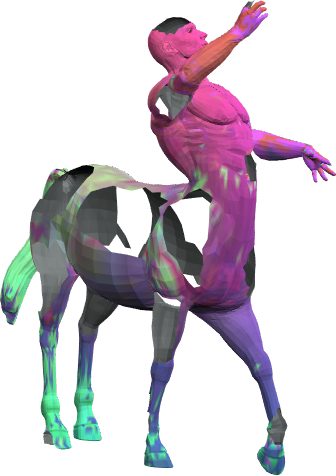} 
    }
    &
    \raisebox{-0.5\height}{
    \includegraphics[width = 0.06\paperwidth]{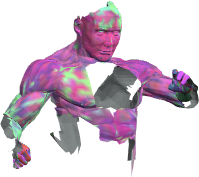} 
    }    
        &
    \raisebox{-0.5\height}{
    \includegraphics[width = 0.06\paperwidth]{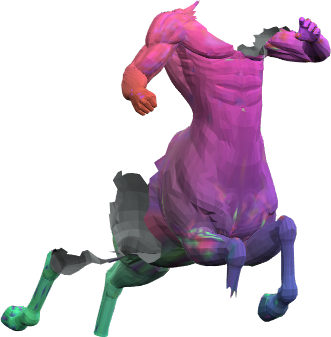} 
    }
    &
    \raisebox{-0.5\height}{
    \includegraphics[width = 0.06\paperwidth]{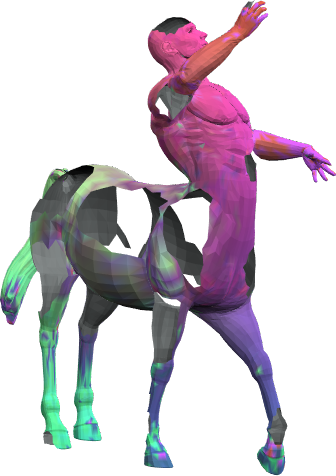} 
    }
    &
    \raisebox{-0.5\height}{
    \includegraphics[width = 0.06\paperwidth]{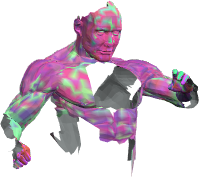} 
    } 
    \\
    \raisebox{-0.5\height}{
    \includegraphics[width = 0.06\paperwidth]{Pictures/render/shrec16holes/color02/holes_horse_shape_7_M.png} 
    }
    &
    \raisebox{-0.5\height}{
    \includegraphics[width = 0.06\paperwidth]{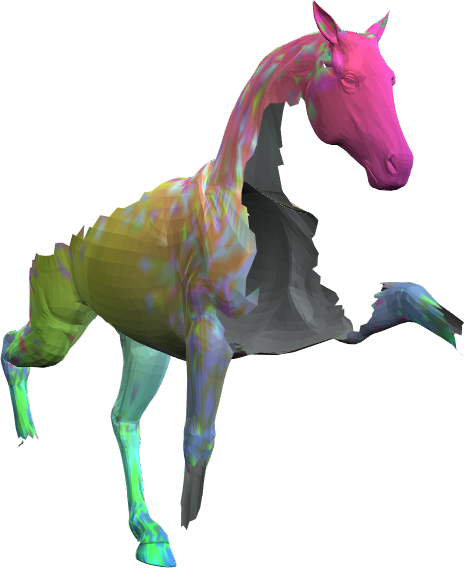} 
    }
    &
    \raisebox{-0.5\height}{
    \includegraphics[width = 0.06\paperwidth]{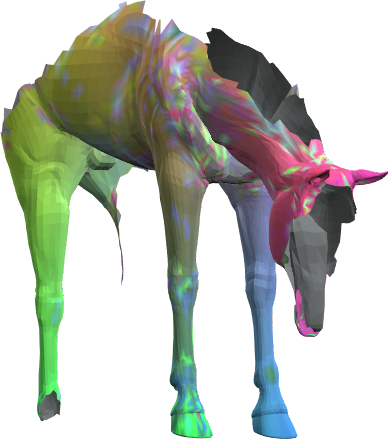} 
    }
    &
    \raisebox{-0.5\height}{
    \includegraphics[width = 0.06\paperwidth]{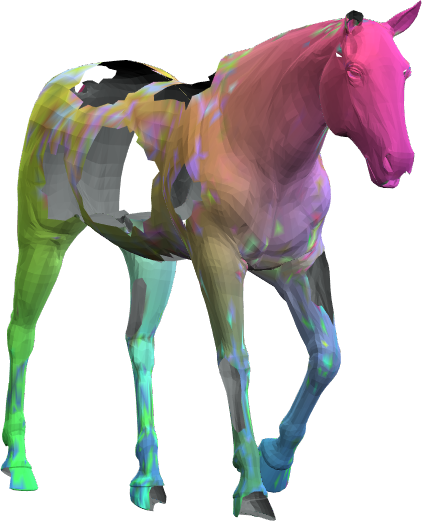} 
    }    
        &
    \raisebox{-0.5\height}{
    \includegraphics[width = 0.06\paperwidth]{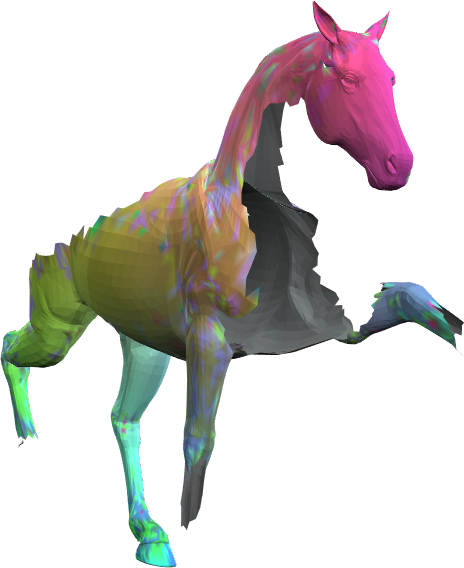} 
    }
    &
    \raisebox{-0.5\height}{
    \includegraphics[width = 0.06\paperwidth]{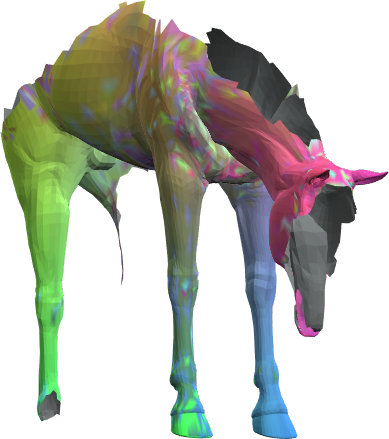} 
    }
    &
    \raisebox{-0.5\height}{
    \includegraphics[width = 0.06\paperwidth]{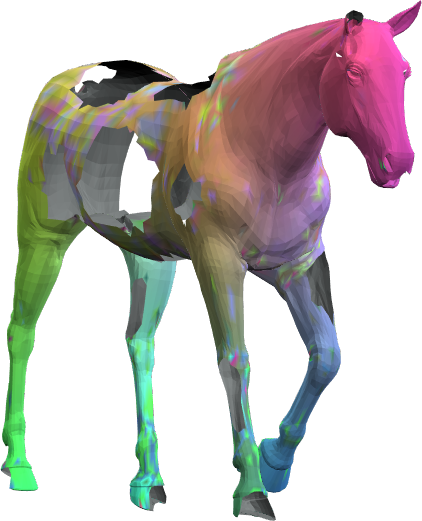} 
    }
    \\
    \raisebox{-0.5\height}{
    \includegraphics[width = 0.06\paperwidth]{Pictures/render/shrec16holes/color02/holes_wolf_shape_7_M.png} 
    }
    &
    \raisebox{-0.5\height}{
    \includegraphics[width = 0.06\paperwidth]{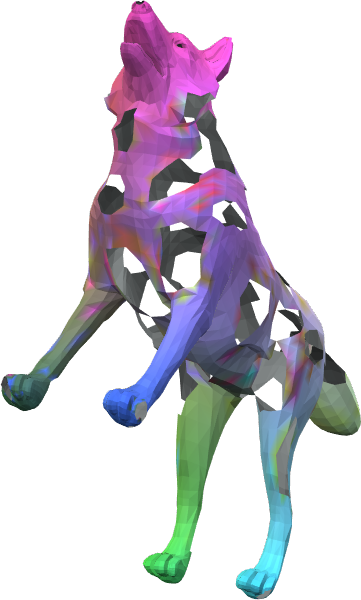} 
    }
    &
    \raisebox{-0.5\height}{
    \includegraphics[width = 0.06\paperwidth]{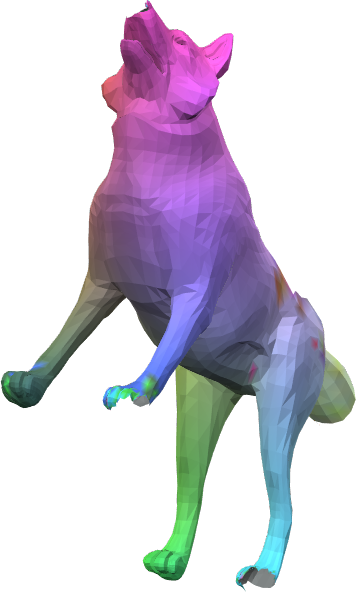} 
    }
    &
    \raisebox{-0.5\height}{
    \includegraphics[width = 0.06\paperwidth]{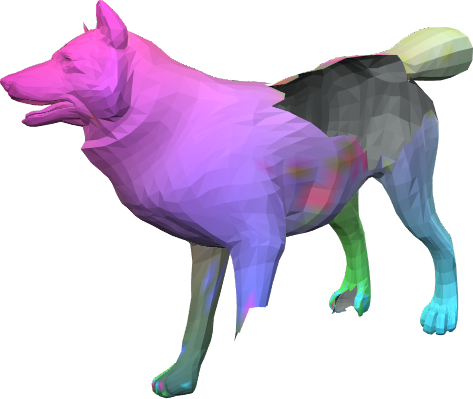} 
    }    
        &
    \raisebox{-0.5\height}{
    \includegraphics[width = 0.06\paperwidth]{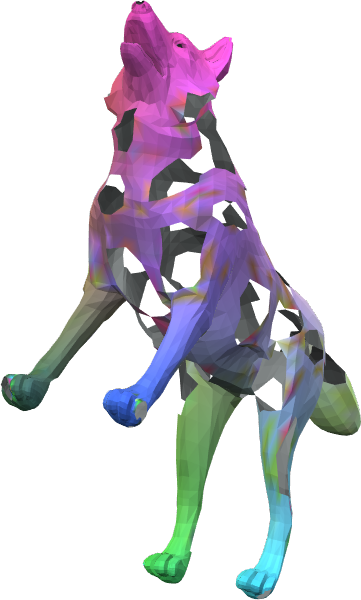} 
    }
    &
    \raisebox{-0.5\height}{
    \includegraphics[width = 0.06\paperwidth]{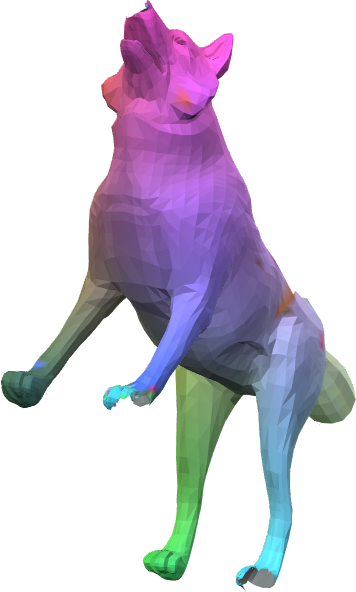} 
    }
    &
    \raisebox{-0.5\height}{
    \includegraphics[width = 0.06\paperwidth]{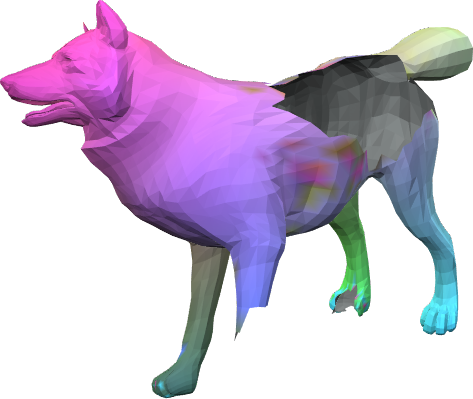} 
    } 
    \\
    \raisebox{-0.5\height}{
    \includegraphics[width = 0.06\paperwidth]{Pictures/render/shrec16holes/color02/holes_cat_shape_7_M.png} 
    }
    &
    \raisebox{-0.5\height}{
    \includegraphics[width = 0.06\paperwidth]{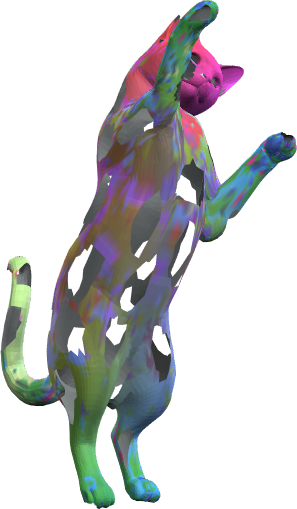} 
    }
    &
    \raisebox{-0.5\height}{
    \includegraphics[width = 0.06\paperwidth]{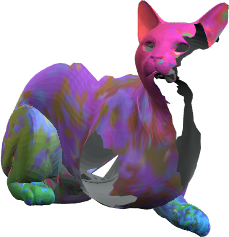} 
    }
    &
    \raisebox{-0.5\height}{
    \includegraphics[width = 0.06\paperwidth]{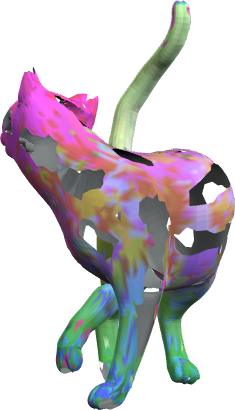} 
    }    
    &
    \raisebox{-0.5\height}{
    \includegraphics[width = 0.06\paperwidth]{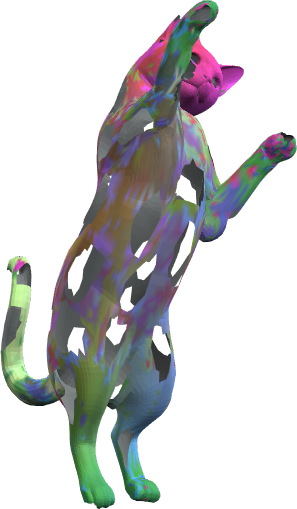} 
    }
    &
    \raisebox{-0.5\height}{
    \includegraphics[width = 0.06\paperwidth]{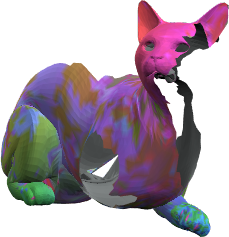} 
    }
    &
    \raisebox{-0.5\height}{
    \includegraphics[width = 0.06\paperwidth]{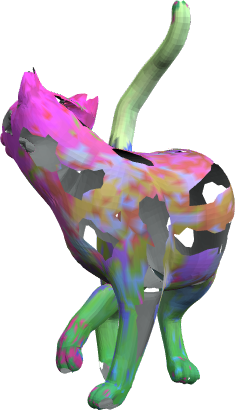} 
    } 
  \end{tabular}
  }
\captionof{figure}{
Visual shape correspondence results on the SHREC'16 Holes dataset. The source shape is on the left, on which we perform dense correspondence estimation on the shapes to the right, using either ACSCNN or our approach. FieldConv fails to provide meaningful correspondence results on partial shapes so we do not display its qualitative performance.}
\label{fig:qual-shrec-holes}
\end{figure*}